\documentclass[11pt]{article}

\usepackage[numbers]{natbib}

\usepackage{booktabs}
\usepackage{amsfonts}
\usepackage{amsmath}
\usepackage{graphicx}
\usepackage{amsthm}
\usepackage{amssymb}
\usepackage{multirow}
\usepackage{makecell}
\usepackage[vlined,linesnumbered,ruled,resetcount]{algorithm2e}
\usepackage[colorlinks,linkcolor=magenta,filecolor=blue,citecolor=blue,urlcolor=blue]{hyperref}
\usepackage[margin=0.5in]{geometry}
\usepackage{subfigure}
\usepackage{pifont}
\usepackage{mathtools}
\usepackage{stmaryrd}
\usepackage[T1]{fontenc}
\usepackage[utf8]{inputenc}
\usepackage{lmodern}

\usepackage{authblk}

\usepackage{enumitem}
\usepackage{algorithmic}
\usepackage{microtype}
\usepackage{xcolor}
\usepackage{bm}
\usepackage{float}
\usepackage{caption}
\usepackage{subcaption}
\usepackage{wrapfig}
\usepackage{tabularx}
\usepackage{url}
\usepackage{thm-restate}
\usepackage{adjustbox}
\usepackage{array}

\newtheorem{theorem}{Theorem}[section]
\newtheorem{lemma}[theorem]{Lemma}
\newtheorem{proposition}[theorem]{Proposition}

\newtheorem{assumption}[theorem]{Assumption}

\theoremstyle{definition}
\newtheorem{definition}[theorem]{Definition}

\theoremstyle{remark}

\DeclareMathOperator{\Cov}{Cov}

\renewcommand{\epsilon}{\varepsilon}
\renewcommand{\phi}{\varphi}

\newcolumntype{R}[2]{%
  >{\adjustbox{angle=#1,lap=\width-(#2)}\bgroup}%
  l%
  <{\egroup}%
}

\title{Forget Without Compromise: Nexus Sampling for Streaming\\ KV-Cache Eviction Under Fixed Budgets}

\author[1]{Duc Duong\thanks{Equal contribution. Contact: \texttt{el72@rice.edu} and \texttt{zhaozhuo.xu@workato.com}.}}
\author[2]{Hoang Anh Duy Le\protect\footnotemark[1]}
\author[4]{Jianwen Xie}
\author[2]{Anshumali Shrivastava}
\author[3]{Zhaozhuo Xu}
\affil[1]{Department of Computer Science, Grinnell College}
\affil[2]{Department of Computer Science, Rice University}
\affil[3]{Workato}
\affil[4]{Lambda, Inc}
\date{}

\begin{document}

\maketitle

\begin{abstract}
  Long-context and agentic LLM workloads push the KV cache past any fixed memory budget, forcing the inference stack to permanently \emph{evict} tokens at every step of a continuous-inference stream.
  Existing methods all share the same template, a per-step direct-attention score followed by \emph{deterministic top-$K$} selection, which converts a single below-cutoff step into an irreversible verdict and permanently erases any subtly important token that direct attention cannot single out from noise.
  To address this challenge, we propose \textbf{Nexus Sampling}, a training-free eviction method that pairs \emph{Nexus scoring}, an iterative walk over direct attention that surfaces bridge tokens, with \emph{weighted reservoir sampling}, which retains tokens with inclusion probability in place of deterministic top-$K$.
  Theoretically, we show that Nexus Sampling dominates deterministic top-$K$ in long-run survival of subtly important tokens.
  Empirically, at $80\%$ KV cache eviction, Nexus Sampling matches dense attention within $\sim$1 point on LongBench while outperforming top-$K$ baselines on retrieval-heavy tasks, with up to $10\times$ smaller per-sequence cache memory.
\end{abstract}

\section{Introduction}
\label{sec:introduction}

GPU memory imposes a hard ceiling on what an LLM inference stack can hold during long-context generation, yet the Key-Value (KV) cache grows linearly with context length and quickly exceeds this budget~\citep{pope2023efficiently,dao2022flashattention}.
The gap widens further in long-lived, agent-style deployments such as multi-turn assistants, persistent reasoning loops~\citep{openai2024o1}, and repository-scale coding agents~\citep{anthropic2025sonnet,jimenez2024swebench}, where the effective context grows toward an unbounded stream.
Sparse attention methods~\citep{tang2024quest,le2026sketchwalk,yuan2025blasst} address this efficiency problem partially by \emph{reading} a still-full cache more efficiently.
But once the cache no longer fits in memory, reading it efficiently is not enough: the inference stack must permanently \emph{evict} tokens.
A growing body of work~\citep{zhang2023h2o,xiao2024streamingllm,cai2024pyramidkv,feng2025adakv,li2024snapkv,ghadia2025morphkv} performs this \emph{KV cache eviction}, retaining the tokens that score highest under the current (or recent) queries' attention and dropping the rest.
Existing methods all share the same per-step design as sparse attention: at every eviction step, they retain the $K$ highest-scoring tokens by \emph{deterministic top-$K$} selection.
We argue, however, that \textbf{eviction is fundamentally different from sparse attention:} while sparse attention performs a one-time read of a still-full cache, eviction is a \emph{streaming, fixed-budget} problem in which tokens arrive continuously, evictions are irreversible, and the policy must repeatedly decide what to retain against future queries it has not yet seen.
Under this view, per-step deterministic top-$K$ is a suboptimal primitive: it treats each step in isolation and inherently confuses transient marginality with permanent unimportance, so a token whose score is only intermittently above the cutoff is dropped on the first marginal step, even though its long-run average importance across steps may be high.

\noindent\textbf{Deterministic top-$K$ misses tokens whose importance fluctuates.}
Per-step importance scores in KV caches are noisy and time-varying.
Attention in modern LLMs is heavy-tailed~\citep{zhang2023h2o,le2026sketchwalk}: a small head of tokens holds most of the attention mass, while the rest of the cache sits at near-uniform low scores where the relative ranking between tokens is dominated by noise (Figure~\ref{fig:heavy_tailed}).
A token important to query $t$ may look marginal at $t\!+\!1$ and important again at $t\!+\!10$.
Top-$K$, however, treats each of these per-step scores as a final verdict: a token that lands below the cutoff at any single step is dropped with the same finality as a token deep in the tail, regardless of how close its score was to the cutoff.
A long context produces many such verdicts, and the per-step errors compound.
The result is \emph{monotone marginal erosion}: every token whose score is only intermittently above the cutoff is silently and permanently lost.

\begin{wrapfigure}{r}{0.5\textwidth}
    \vspace{-1em}
    \centering
    \includegraphics[width=\linewidth]{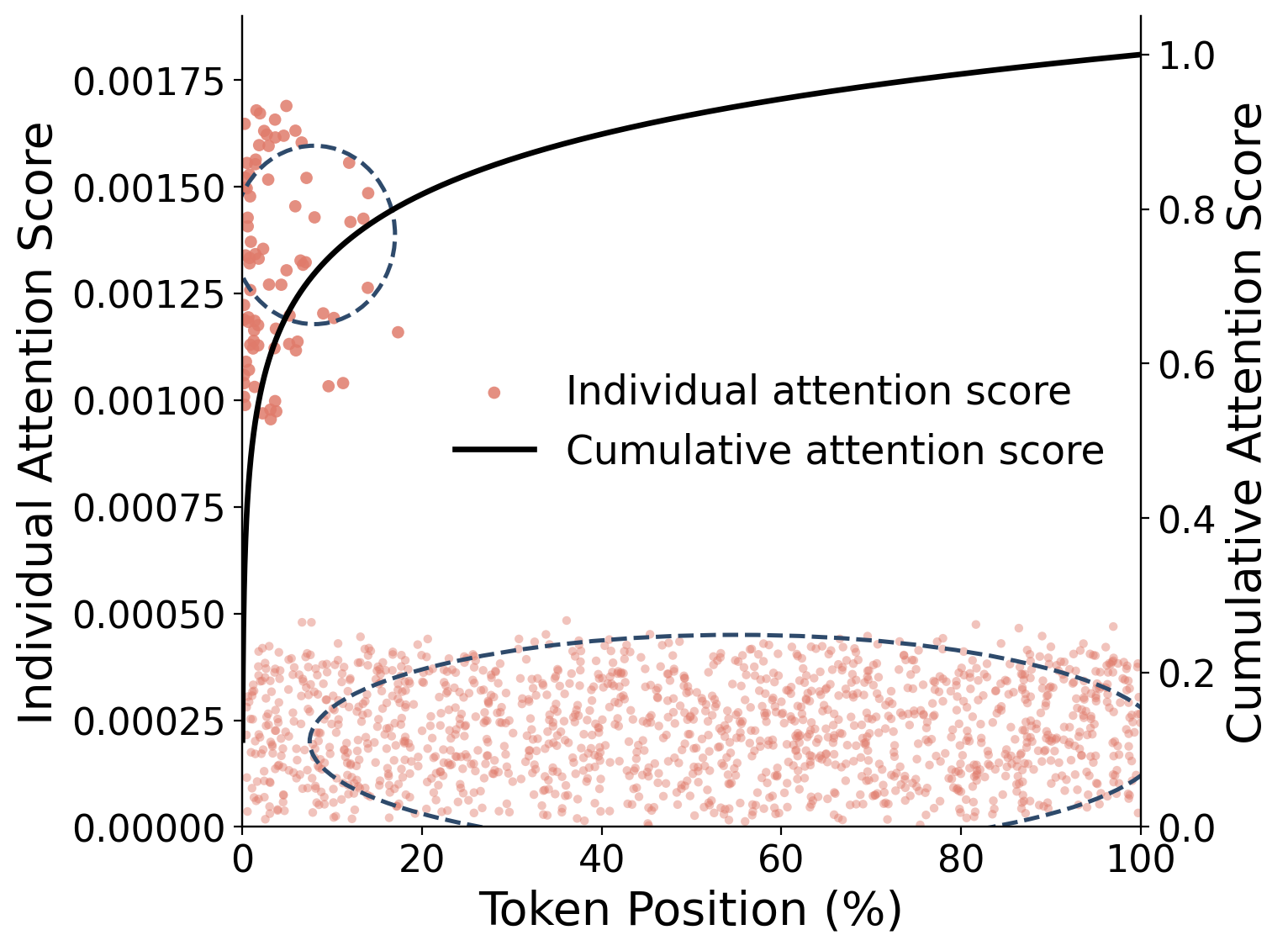}
    \vspace{-1em}
    \caption{Attention scores are heavy-tailed: $\sim$80\% of cumulative mass concentrates on the first $\sim$15\% of token positions (black curve), with a small head of high-score tokens (top-left cluster) and a long, noisy tail of marginal-score tokens (bottom band) where deterministic top-$K$ eviction is forced to choose between near-identical scores.}
    \label{fig:heavy_tailed}
\end{wrapfigure}

\noindent\textbf{Reservoir sampling is the streaming-algorithms answer to this setting.}
Weighted reservoir sampling~\citep{vitter1985random,efraimidis2006weighted} retains each token with probability proportional to its current score, so that a token's long-run survival across many steps tracks its \emph{average importance across steps} rather than collapsing on the first step where its score happens to land below the cutoff.
We believe this is the right candidate of selection primitive for KV cache eviction, but it inherits the calibration of the score it samples against.
The direct-attention magnitude used by existing methods is, however, locally myopic: it counts only what the current query attends to, and misses what we call \emph{bridge tokens} - tokens that no single recent query attends to strongly, but that hold together a strongly-connected cluster of mutually-attended tokens across the context, such that removing one would sever the cluster's internal connections.
Bridges are a canonical instance of marginal-rank tokens: deterministic top-$K$ evicts them at every step, even though the clusters they hold together may be highly important.
We therefore lift the reservoir's input weight from this direct-attention score to a \emph{Nexus score} that surfaces bridges before any selection happens.

We propose \textbf{Nexus Sampling}, a training-free KV cache eviction method built around two new mechanisms that together address both failure modes identified above, applied uniformly across prefill and decode.
Upstream, \emph{Nexus scoring} passes the direct-attention score through a short iterative walk recurrence that surfaces bridge tokens before any selection happens.
Downstream, \emph{weighted reservoir sampling} draws the $K$ retained tokens with inclusion probability proportional to this walk-augmented weight, replacing the deterministic top-$K$ selection step that every prior eviction method shares.
Each fix recovers a category of importance the existing template cannot see: the walk recovers bridges that direct-attention top-$K$ misses, and the reservoir recovers the marginal-rank tokens that any deterministic top-$K$ erodes.
Theoretically, Nexus's long-run token survival decays as a product of per-step inclusion probabilities (a mean over steps), where deterministic top-$K$'s collapses to zero on the first below-cutoff step (details in App.~\ref{sec:theory}).
The contributions are:

\begin{itemize}[leftmargin=*,nosep]
    \item We identify a fundamental limitation of deterministic-top-$K$ KV cache eviction: it converts transient marginality into permanent loss, producing \emph{monotone marginal erosion} in the streaming, fixed-budget regime that is inherent to KV cache eviction.
    \item We propose \textbf{Nexus Sampling}, a training-free KV cache eviction method that applies uniformly to prefill and decode, combining weighted reservoir sampling with a \emph{Nexus score} that surfaces \emph{bridge tokens}, tokens anchoring strongly-connected clusters of mutually-attended tokens across the window.
    \item We establish theoretical guarantees showing that reservoir sampling's long-run token survival is a \emph{product over steps} where deterministic top-$K$'s is a \emph{min over steps}, and empirically demonstrate at $20\%$ density that Nexus Sampling stays within $\sim$1 point of dense attention on LongBench while outperforming top-$K$ baselines on retrieval-heavy long-context tasks, with up to $10\times$ smaller per-sequence cache memory than dense FlashAttention-2.
\end{itemize}

\section{Nexus Sampling}
\label{sec:method}

We introduce \textbf{Nexus Sampling}, a KV cache eviction method that runs alongside decoding: at every eviction step $t$ (whenever the budget is hit during the decoding stream), Nexus produces the retention set $S_t$ via two components applied in sequence: \emph{Nexus scoring} followed by \emph{weighted reservoir selection}.
Nexus scoring passes the per-block attention score through a short iterative walk recurrence that surfaces \emph{bridge tokens}, tokens anchoring strongly-connected clusters of mutually-attended tokens, and yields a per-block sampling weight $\mathbf{w}^{(t)}$ that reflects indirect importance.
The reservoir step then draws the $K$ retained blocks with inclusion probability proportional to $\mathbf{w}^{(t)}$ under the weighted-without-replacement law, replacing the deterministic top-$K$ selection that every prior eviction method shares.
The name reflects the two roles each component plays at every step in the stream: Nexus scoring identifies the \emph{nexus} blocks that hold a cluster together, and the reservoir \emph{samples} from the resulting weights so that every positive-weight block keeps positive inclusion probability rather than being deterministically condemned the first time it lands at the marginal rank.

\subsection{Preliminaries and Notation}
\label{sec:prelim}

\noindent\textbf{Streaming view.}
We treat KV cache eviction as a streaming problem over the steps of inference.
At each step $t$, the LLM ingests one new query $q_t \in \mathbb{R}^{1 \times D}$ (a prompt token during prefill or a decoded token during decode), attends over the current cache $K_t, V_t$, and appends one new (key, value) pair, after which the cache may exceed the memory budget.
An \emph{eviction step} occurs whenever the budget is hit: a subset $S_t \subseteq \{1,\dots,T_t\}$ of $B$ positions is retained, the rest are dropped, and inference continues with $K_{S_t}, V_{S_t}$.
Eviction is therefore relevant in both phases: a long prompt can trigger a one-shot eviction at the end of prefill, while a long generation produces a long stream of eviction steps throughout decode.
In both cases, eviction is irreversible: positions evicted at step $t$ are not recoverable at any step $t' > t$.

\noindent\textbf{Block granularity.}
We work at the granularity of \emph{key blocks} rather than individual tokens, amortizing the per-decision cost over $b$ tokens ($b=32$); blocks are indexed by $j \in \{1,\dots,N_k\}$ with $N_k = \lceil T_k / b \rceil$ where $T_k$ is the number of cached keys.

\noindent\textbf{Per-block base score.}
Single-query attention is noisy and dominated by the immediate token, so following standard practice~\citep{li2024snapkv,ghadia2025morphkv} we score blocks against a length-$W$ \emph{observation window} of recent queries $\mathbf{Q} \in \mathbb{R}^{W \times D}$ on head-averaged representations.
The per-block \emph{base score} $\mathbf{a} \in \mathbb{R}^{N_k}$ is the row average of $\widehat{\mathbf{S}} = \operatorname{rownorm}(\operatorname{BlockSum}_b(\operatorname{softmax}(\mathbf{Q}\mathbf{K}^\top / \sqrt{D})))$:
\begin{equation}
    \mathbf{a} = \tfrac{1}{W} \textstyle\sum_{w=1}^{W} \widehat{\mathbf{S}}_{w,\cdot},\quad \textstyle\sum_j a_j = 1,
    \label{eq:base-score}
\end{equation}

\noindent\textbf{Notation.}
Throughout the rest of the section we use $W$ for the observation-window length, $\widehat{\mathbf{S}} \in \mathbb{R}^{W \times N_k}$ for the per-row block distribution, $\mathbf{a} \in \mathbb{R}^{N_k}$ for the base score (Eq.~\ref{eq:base-score}), $\mathbf{c} \in \mathbb{R}^{N_k}$ for the walk state, $\mathbf{w} \in \mathbb{R}^{N_k}$ for the combined sampling weight, $B$ for the total retained-block budget, and $n$ for the reservoir averaging count.
A full notation table appears in App.~\ref{app:notation}.

\begin{figure*}[t]
  \centering
  \includegraphics[width=0.85\linewidth]{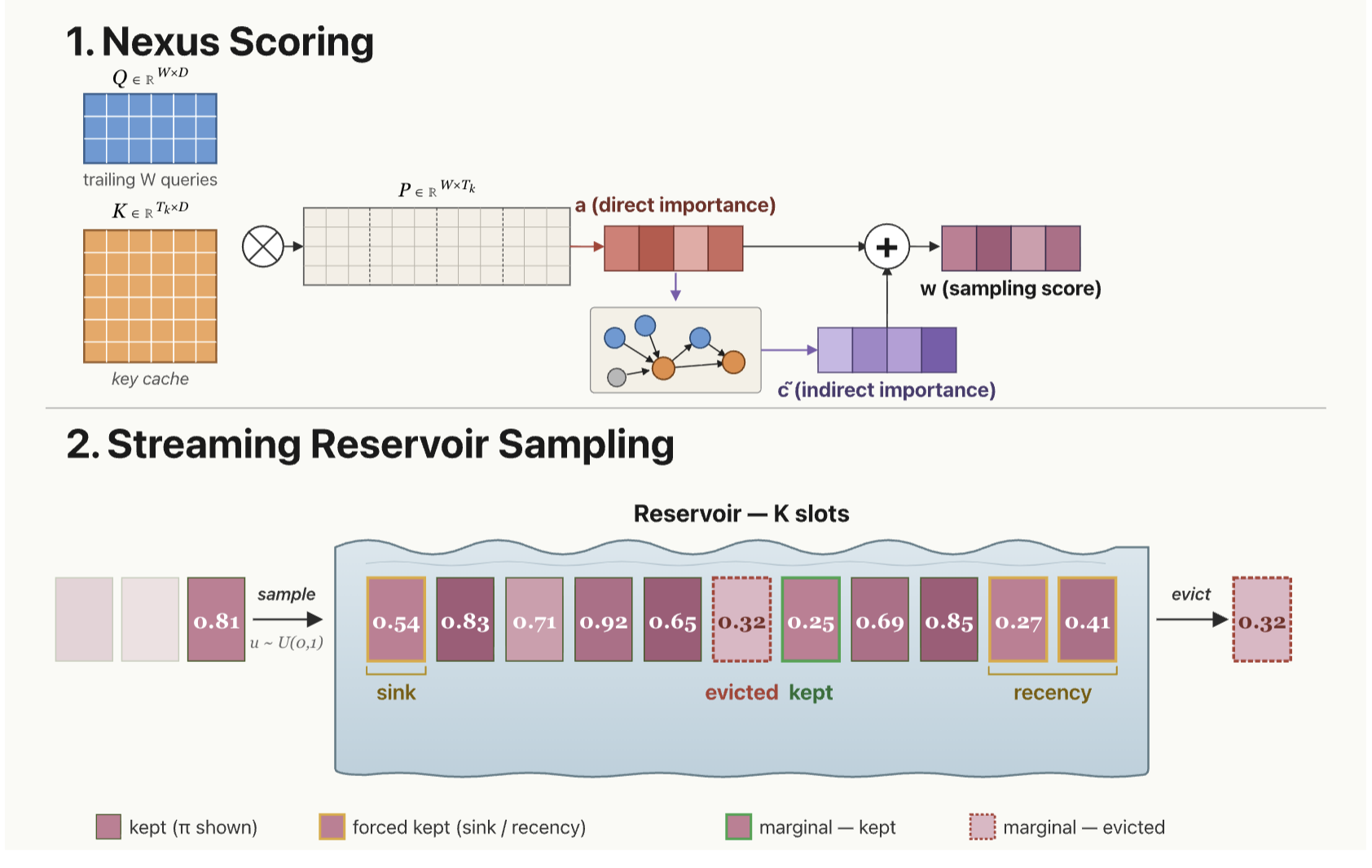}
  \caption{Overview of Nexus Sampling at one eviction step.
    \emph{Nexus scoring} (top) combines a direct importance term (windowed per-query attention) with a bridge term obtained by an iterative walk over the same query rows.
    \emph{Streaming reservoir sampling} (bottom) then draws the retained blocks with inclusion probability proportional to the resulting weight, replacing the deterministic top-$K$ used by every prior eviction method.}
  \label{fig:nexus-overview}
\end{figure*}

\subsection{Nexus Scoring}
\label{sec:method-walk}

Nexus scoring assigns each block a two-term \emph{Nexus score}
\begin{equation}
  s_j = a_j + \lambda\,\tilde{c}_j,
  \label{eq:nexus-score}
\end{equation}
that captures both the \emph{direct} importance $a_j$ of block $j$ (how much the recent queries attend to it) and an \emph{indirect} importance $\tilde{c}_j$ that measures how strongly block $j$ acts as a \emph{bridge token}, i.e., how strongly it anchors a cluster of mutually-attended tokens across the window even when no single recent query attends to it strongly.
The two terms cover complementary failure modes of a single-query attention score: $a_j$ scores blocks the current query has reason to look at; $\tilde{c}_j$ scores blocks the current query has reason to \emph{remember}, by virtue of their role in the recent context.
The mixing weight $\lambda \geq 0$ balances the two terms.

Both terms are computed from the same per-query rows of $\widehat{\mathbf{S}}$.
Denote the rows by $\mathbf{a}^{(q)} \in \mathbb{R}^{N_k}$, so $\mathbf{a}^{(q)} = \widehat{\mathbf{S}}_{q,\cdot}$ is the block distribution of window query $q$.
The \emph{direct} importance is the uniform row-average,
\begin{equation}
  \mathbf{a} \;=\; \tfrac{1}{W} \textstyle\sum_{q=1}^{W} \mathbf{a}^{(q)},
  \label{eq:direct}
\end{equation}
which treats each window query as an equal vote.
The \emph{bridge} score, by contrast, aggregates the same rows through an alignment-weighted iterative walk.
Starting from $\mathbf{C} = \mathbf{0}$, the walk runs for $H$ steps; at each step we plug in a fresh window query row $\mathbf{a}^{(q)}$ and update
\begin{equation}
  \mathbf{C} \;\leftarrow\; \mathbf{C} + \gamma_q\,\mathbf{a}^{(q)},\qquad
  \gamma_q = 1 + \langle \mathbf{a}^{(q)},\,\mathbf{C}\rangle,
  \label{eq:walk}
\end{equation}
from which $\tilde{\mathbf{c}} = \mathbf{C} / \|\mathbf{C}\|_1$.
The per-block weight fed to the reservoir is the Nexus score plus a small recency tie-break,
\begin{equation}
  w_j = s_j + \varepsilon_{\text{tie}}\,r_j,\quad r_j = j/(N_k - 1),
  \label{eq:combined}
\end{equation}
with $r_j$ a linear ramp from $0$ (oldest block) to $1$ (newest), and $\varepsilon_{\text{tie}}$ small enough that it does not flip any real signal but large enough to deterministically break ties in favor of more-recent blocks.

\subsection{Weighted Reservoir Block Selection}
\label{sec:method-ares}

Given the per-block weight $\mathbf{w} \in \mathbb{R}^{N_k}$ from \S\ref{sec:method-walk}, the reservoir step decides which $K$ of the $N_k$ blocks survive the eviction step.
The choice of selection rule here, not the choice of weight, is what distinguishes Nexus Sampling from every prior eviction method.
Following standard practice, we retain the attention-sink block and a small recency window unconditionally; these account for $B_{\text{forced}}$ slots of the total budget $B$, leaving $K = B - B_{\text{forced}}$ slots for the reservoir step below.

\noindent\textbf{$n$-averaged reservoir priority.}
For every remaining candidate block $j$, we draw $n$ independent uniform variates $u_j^{(1)},\dots,u_j^{(n)} \sim \mathcal{U}(0,1)$ and form the \emph{$n$-averaged reservoir priority}
\begin{equation}
  \pi_j = \frac{1}{n}\sum_{i=1}^{n} \bigl(u_j^{(i)}\bigr)^{1/w_j}.
  \label{eq:ares}
\end{equation}
We then keep the top $K$ candidates by $\pi_j$.

\section{Why Nexus Sampling for KV Cache Eviction?}
\label{sec:why}

Every existing eviction method makes two implicit choices: that a token's importance is read off from \emph{direct} per-step attention, and that the $K$ surviving tokens are picked by \emph{deterministic top-$K$}.
Both choices are defensible for a one-shot prefill against a still-dense cache, but both might break in the streaming, fixed-budget regime that defines KV cache eviction (\S\ref{sec:method}).
Nexus Sampling replaces them with two components, Nexus scoring and weighted reservoir selection, each addressing a failure mode that the other cannot reach on its own.

\noindent\textbf{Deterministic top-$K$ erases subtly important tokens at the cutoff; the reservoir preserves them.}
A streaming eviction policy is applied at every step, and under deterministic top-$K$ a token at rank $K\!+\!1$ is dropped with the same finality as a token at rank $10^6$, so a token survives the first $e$ steps only if its weight clears the cutoff at \emph{every} step.
Long-run survival is thus a min over steps and collapses to zero on the first below-cutoff step, permanently erasing any token whose per-step weight sits just below the cutoff, even if it is subtly important on average and would be relevant later in the stream.
Weighted reservoir sampling~\citep{efraimidis2006weighted} replaces this min with a product: per-step inclusion probability is monotone in the weight and strictly positive whenever the weight is, so long-run survival
\begin{equation*}
    p_j^{(e)} \;=\; \prod_{e' \leq e} \Pr\!\bigl(j \in \mathrm{top}\text{-}K\text{ by }\pi^{(e')}\bigr)
\end{equation*}
stays bounded away from zero whenever $w_j^{(e')}$ does (\S\ref{sec:method-ares}; Lemma~\ref{lem:streaming}).
The reservoir thus fixes the \emph{selection rule}: it converts transient marginality into a controlled per-step probability rather than an irreversible verdict, and exactly the subtly important tokens the deterministic rule would erase are the ones it keeps at a controlled non-zero rate.

\noindent\textbf{Direct attention misses bridge tokens; Nexus scoring lifts them.}
\emph{Bridge tokens} are a canonical instance of the subtly important tokens above: tokens that no single recent query attends to \emph{strongly}, but that sit at the intersection of many mutually-attended tokens in the window and structurally hold their cluster together.
The direct score $\mathbf{a}$ is a uniform average of per-query distributions and does not reward this structural role: a bridge typically picks up moderate mass from many queries but rarely a peak from any one, so it lands in the noisy mid-band of the heavy-tailed score distribution~\citep{zhang2023h2o,le2026sketchwalk} rather than in the head.
Even a perfect selection rule cannot recover such a token from this score, because sampling (or top-$K$) proportional to a weight that does not see the structural signal samples from the wrong distribution.
The bridge term $\tilde{c}_j$ in the Nexus score (\S\ref{sec:method-walk}) is built precisely to extract this higher-order signal: an iterative walk over the same per-query rows compounds mass on tokens that successive queries \emph{agree on}, amplifying exactly the cluster-anchoring positions a uniform average flattens out, which we formalize as \emph{hub amplification} (Theorem~\ref{thm:main}).
The walk thus fixes the \emph{weight construction}: it ensures the input to selection already reflects indirect importance, so the reservoir's preservation guarantee applies to bridges rather than to a wrong notion of importance.

\noindent\textbf{Both fixes are needed, and they compose.}
The two components are not redundant; each guards against a failure mode the other cannot.
The reservoir alone cannot save bridge tokens: if they are not lifted into the head of the weight, sampling proportional to weight samples from the wrong distribution.
The walk alone cannot save subtly important tokens that still sit at the cutoff after scoring: a comprehensive per-step weight is still terminated by deterministic top-$K$ the first time it lands below the cutoff.
Bridge tokens are recovered upstream in the weight; near-cutoff subtly important tokens are recovered downstream in the selection; together they remove both of the implicit assumptions above.
The empirical picture in \S\ref{sec:experiments} matches: existing methods do well when the score has a clear head and the relevant tokens are stably top-ranked, and existing methods degrade on multi-hop retrieval, long summarization, and tasks whose answer depends on subtly important tokens that the recent query does not directly attend to.

\section{Theoretical Guarantees}
\label{sec:theory-body}

The arguments of \S\ref{sec:why} have formal counterparts.
The reservoir's long-run survival decays as a \emph{product} over eviction steps where deterministic top-$K$'s collapses to zero on the first below-cutoff step (Lemma~\ref{lem:streaming}), the walk \emph{provably amplifies} hub tokens above the uniform window average (Theorem~\ref{thm:main}), and the two together yield an end-to-end bound on per-step eviction quality (Proposition~\ref{prop:eviction-quality}).
Proofs and supporting lemmas are deferred to App.~\ref{sec:theory}.

\begin{lemma}[Min over steps vs.\ product over steps]
  \label{lem:streaming}
  Consider a sequence of eviction steps $e = 1, \dots, E$ with per-step weights $w_j^{(e)} > 0$ at every step, and let $w_{(K)}^{(e)}$ denote the $K$-th largest weight at step $e$.
  Then the long-run survival probabilities of block $j$ under deterministic top-$K$ and weighted reservoir sampling satisfy
  \begin{align*}
    S_j^{\text{top-}K}(E) & \;=\; \textstyle\prod_{e=1}^{E} \mathbf{1}\!\bigl[w_j^{(e)} \geq w_{(K)}^{(e)}\bigr], \\
    S_j^{\text{res}}(E)   & \;=\; \textstyle\prod_{e=1}^{E} q_j^{(e)},
  \end{align*}
  where $q_j^{(e)} \in (0,1]$ is strictly positive whenever $w_j^{(e)} > 0$.
  The first product collapses to zero as soon as any single step has $w_j^{(e)} < w_{(K)}^{(e)}$; the second stays bounded away from zero as long as $w_j^{(e)}$ does.
\end{lemma}

Lemma~\ref{lem:streaming} formalizes the min-vs-product contrast from \S\ref{sec:why}: deterministic top-$K$ converts a single below-cutoff step into an irreversible verdict, while reservoir sampling decays survival only as the product of per-step inclusion probabilities, exactly the property that lets subtly important tokens at the cutoff have a chance to persist across a long stream.

\begin{theorem}[Hub amplification]
  \label{thm:main}
  Assume the $H$ walk-step rows $\mathbf{a}^{(1)},\dots,\mathbf{a}^{(H)}$ are i.i.d.\ with mean $\mathbf{a}_\star$ and covariance $\boldsymbol{\Sigma}$, and define the \emph{hub score} $h_j = \boldsymbol{\Sigma}_j^\top \mathbf{a}_\star$.
  The walk recurrence~\eqref{eq:walk} satisfies
  \begin{equation*}
    \mathbb{E}[\tilde{c}_j^{(H)}] \;=\; a_{\star,j} \;+\; \tfrac{D_H}{B_H}\,h_j \;+\; O(\|\boldsymbol{\Sigma}\|_F^2),
  \end{equation*}
  with $B_H \geq H$ and $D_H \geq H(H-1)/2$ both block-independent.
  In particular, $\mathbb{E}[\tilde{c}_j^{(H)}] \gtrless a_{\star,j} \iff h_j \gtrless 0$: hub blocks are amplified above the unwalked score and peripheral blocks are attenuated.
\end{theorem}

Theorem~\ref{thm:main} formalizes the bridge-recovery claim of \S\ref{sec:why}: the walk does not merely denoise the direct score, it injects a block-dependent term proportional to the hub score $h_j$, which is exactly the higher-order signal a uniform average $\mathbf{a}$ over per-query rows cannot recover.
Because reservoir inclusion is monotone in the sampling weight (App.~\ref{app:ares}), the hub amplification carries through to the selection step: bridge-like tokens that the direct score $\mathbf{a}$ leaves in the noisy mid-band are lifted into the head of the Nexus score, and the reservoir then keeps them at the controlled non-zero rate of Lemma~\ref{lem:streaming}.
Lemma~\ref{lem:streaming} and Theorem~\ref{thm:main} are component-level guarantees on selection and scoring; we close the loop by combining them into an end-to-end bound on the retained-utility error of one eviction step.
Let $z_j \geq 0$ denote the utility that block $j$ would contribute to future attention if it remained in cache (e.g., its future block-attention mass over a short horizon).

\begin{proposition}[Eviction-quality bound]
  \label{prop:eviction-quality}
  Fix one eviction step and let $I_j$ be the reservoir inclusion indicator with $p_j = \Pr(I_j = 1)$.
  The Horvitz--Thompson estimator $\widehat{Z}_{\mathrm{HT}} = \sum_j I_j z_j / p_j$ is unbiased for $Z = \sum_j z_j$, and if $p_j \geq p_{\min} > 0$ and $0 \leq z_j \leq Z_{\max}$, then with probability $\geq 1 - \delta$,
  \begin{equation*}
    |\widehat{Z}_{\mathrm{HT}} - Z| \;\leq\; Z_{\max}\sqrt{K N_k / (\delta\, p_{\min})}.
  \end{equation*}
  Moreover, if the Nexus weight approximates future utility as $\|\mathbf{w} - \mathbf{z}\|_\infty \leq \eta$, then the expected evicted utility $L = \sum_j (1 - I_j) z_j$ satisfies
  \begin{equation*}
    \bigl|\mathbb{E}[L] - \textstyle\sum_j (1 - p_j) w_j\bigr| \leq \eta N_k.
  \end{equation*}
\end{proposition}

Proposition~\ref{prop:eviction-quality} separates the two sources of approximation: the reservoir step (Lemma~\ref{lem:streaming}) contributes no systematic bias once inclusion probabilities are accounted for, so the remaining error is controlled by how well the Nexus weight $\mathbf{w}$ tracks the future utility $\mathbf{z}$, which is precisely what the walk (Theorem~\ref{thm:main}) is designed to improve over the direct score $\mathbf{a}$.

\section{Experiments}
\label{sec:experiments}

We evaluate Nexus Sampling on two long-context benchmarks under two eviction regimes, comparing against competitive baselines in each regime.

\subsection{Setup}
\label{sec:exp-setup}

\noindent\textbf{Models.}
We evaluate three long-context instruction-tuned models at different scales: Llama-3.1-8B-Instruct, Llama-3.2-1B-Instruct, and Qwen3-8B, all supporting context lengths up to 128K tokens, with the default chat template for each model.
\textbf{Benchmarks.}
We evaluate on two complementary long-context benchmarks: (i)~\textit{\textbf{LongBench}}~\citep{bai2024longbench}, spanning single- and multi-document QA, summarization, to few-shot tasks; and (ii)~\textit{\textbf{RULER}}~\citep{hsieh2024ruler}, a synthetic diagnostic stressing retrieval and position-sensitive reasoning over very long contexts (4K--64K tokens).
On RULER we report average accuracy across the 13-task suite at each context length.
\textbf{Baselines.}
We compare Nexus Sampling against \textit{\textbf{SnapKV}}~\citep{li2024snapkv} and \textit{\textbf{PyramidKV}}~\citep{cai2024pyramidkv} in the prefill-only setting and against \textit{\textbf{H2O}}~\citep{zhang2023h2o} and \textit{\textbf{MorphKV}}~\citep{ghadia2025morphkv} in the prefill\,+\,decode setting, with \textit{\textbf{Full Attention}} as an upper-reference. We additionally report Ada-KV variants (Ada-SnapKV, Ada-PyramidKV), which
layer the head-adaptive budget allocation of \citet{feng2025adakv} on top of
each score; see Appendix~\ref{app:adakv} for the isolated comparison.
All methods are at $20\%$ density (80\% eviction). 

\textbf{Implementation details.}
We implement Nexus Sampling on top of HuggingFace Transformers, with block-wise eviction at block size $b = 32$ and dense attention on the retained cache served by FlashAttention-2 kernels~\citep{dao2023flashattention2}.
Observation window $W = 16$; walk depth $H = 3$; mixing weight $\lambda = 0.5$; tie-break magnitude $\varepsilon_{\text{tie}} = 10^{-6}$; reservoir averaging $n = 5$. Because the reservoir step is stochastic, all reported numbers use a fixed
random seed; the $n=5$ priority averaging (Eq.~\ref{eq:ares}) further
suppresses per-seed variation, which we find negligible relative to the gaps
in Tables~\ref{tab:longbench}--\ref{tab:ruler}.
All experiments run on a single NVIDIA H200 (143\,GB) with greedy decoding.

\subsection{LongBench Results}
\label{sec:exp-longbench}

\begin{table*}[ht]
    \centering
    \small
    \caption{LongBench accuracy at 80\% KV cache eviction ($B/N_k = 0.20$, i.e., \emph{Density = 20\%}) across three models and 16 tasks.
        \emph{Prefill} blocks evict only during the prefill phase (baselines: SnapKV, PyramidKV); \emph{Prefill\,+\,Decode} blocks evict throughout the entire inference stream (baselines: H2O, MorphKV).
        \textbf{Bold} marks the leading method per row across the eviction methods (Full Attention shown for reference).}
    \label{tab:longbench}
    \setlength{\tabcolsep}{2.3pt}
    \resizebox{\linewidth}{!}{%
        \begin{tabular}{ll c | cccccccc cccccccc | c}
            \toprule
            \textbf{Model}                & \textbf{Method}         & \textbf{Density} & \textbf{WIKI}  & \textbf{GOV}   & \textbf{HPQA}  & \textbf{LCC}   & \textbf{MNews} & \textbf{MFQA}  & \textbf{MUS}   & \textbf{NQA}   & \textbf{COUNT} & \textbf{Retr.} & \textbf{QAS}   & \textbf{QMS}   & \textbf{REPO}  & \textbf{SamS}  & \textbf{TREC}  & \textbf{TRIV}  & \textbf{Avg.}  \\
            \midrule
            \multicolumn{19}{l}{\emph{Prefill-only eviction.}}                                                                                                                                                                                                                                                                                                                          \\
            \midrule
            \multirow{4}{*}{Llama-3.1-8B} & Full Attention          & 100\%             & 47.75          & 34.70          & 56.96          & 55.20          & 26.76          & 56.16          & 32.77          & 29.91          & 9.66           & 99.50          & 45.06          & 25.39          & 47.80          & 43.27          & 73.00          & 92.14          & 48.50          \\
                                          & SnapKV                  & 20\%              & 47.17          & 28.69          & \textbf{58.61} & \textbf{55.42} & \textbf{23.19} & 55.44          & 31.53          & \textbf{30.28} & \textbf{10.25} & \textbf{99.50} & 40.03          & 24.60          & \textbf{48.75} & 42.17          & 68.50          & 91.43          & 47.22          \\
                                          & PyramidKV               & 20\%              & 47.96          & 28.07          & 57.33          & 53.46          & 22.37          & 56.11          & \textbf{32.97} & \textbf{30.49} & \textbf{10.33} & \textbf{99.50} & 42.36          & 24.90          & 47.02          & 42.53          & 71.00          & 91.68          & 47.38          \\
                                          & \textbf{Nexus Sampling} & 20\%              & \textbf{48.82} & \textbf{30.07} & 58.03          & 54.62          & 22.65          & \textbf{57.83} & 32.68          & 29.24          & 10.00          & 99.00          & \textbf{43.47} & \textbf{25.24} & 47.87          & \textbf{42.93} & \textbf{73.00} & \textbf{92.66} & \textbf{48.01} \\
            \midrule
            \multirow{4}{*}{Llama-3.2-1B} & Full Attention          & 100\%             & 31.23          & 29.65          & 35.47          & 29.83          & 25.88          & 43.35          & 18.45          & 20.40          & 3.67           & 4.50           & 16.35          & 21.84          & 36.04          & 39.77          & 61.50          & 78.54          & 31.03          \\
                                          & SnapKV                  & 20\%              & 30.45          & 22.92          & \textbf{35.48} & 30.75          & 19.61          & 39.04          & 17.15          & \textbf{21.58} & 3.67           & \textbf{5.00}  & 15.06          & 21.18          & 35.86          & 37.40          & 58.00          & 79.31          & 29.53          \\
                                          & PyramidKV               & 20\%              & 28.40          & 20.27          & 34.89          & 28.62          & 16.92          & 39.91          & 16.19          & 19.63          & 3.00           & \textbf{5.00}  & 14.47          & 20.80          & 35.13          & 37.37          & 58.00          & 78.45          & 28.57          \\
                                          & \textbf{Nexus Sampling} & 20\%              & \textbf{31.86} & \textbf{24.58} & 34.52          & \textbf{30.79} & \textbf{20.91} & \textbf{41.87} & \textbf{17.53} & 20.75          & \textbf{5.67}  & \textbf{5.00}  & \textbf{15.30} & \textbf{21.74} & \textbf{36.23} & \textbf{37.71} & \textbf{58.50} & \textbf{79.95} & \textbf{30.18} \\
            \midrule
            \multirow{4}{*}{Qwen3-8B}     & Full Attention          & 100\%             & 42.62          & 33.58          & 57.94          & 57.39          & 24.82          & 52.94          & 34.33          & 27.64          & 4.50           & 100.00         & 47.91          & 23.85          & 56.67          & 44.17          & 71.50          & 90.71          & 48.16          \\
                                          & SnapKV                  & 20\%              & 42.38          & 28.76          & 56.94          & \textbf{58.07} & 20.29          & \textbf{52.76} & \textbf{33.92} & \textbf{27.96} & 6.00           & \textbf{100.00}& 43.55          & 23.17          & \textbf{56.60} & \textbf{43.90} & 66.50 & 90.71          & 46.97          \\
                                          & PyramidKV               & 20\%              & 39.56          & 26.87          & 55.25          & 55.24          & 18.44          & 48.48          & 31.78          & 26.02          & \textbf{6.50}  & \textbf{100.00}& 36.61          & 22.29          & 53.29          & 43.40          & 68.00          & 90.55          & 45.14          \\
                                          & \textbf{Nexus Sampling} & 20\%              & \textbf{42.49} & \textbf{31.73} & \textbf{57.68} & 56.93          & \textbf{22.11} & 52.52          & 33.73          & 27.44          & 5.00           & 99.50          & \textbf{43.72} & \textbf{23.89} & 56.23          & 43.85          & \textbf{71.00} & \textbf{91.36} & \textbf{47.45} \\
            \midrule
            \multicolumn{19}{l}{\emph{Prefill\,+\,decode eviction.}}                                                                                                                                                                                                                                                                                                                    \\
            \midrule
            \multirow{4}{*}{Llama-3.1-8B} & Full Attention          & 100\%             & 47.75          & 34.70          & 56.96          & 55.20          & 26.76          & 56.16          & 32.77          & 29.91          & 9.66           & 99.50          & 45.06          & 25.39          & 47.80          & 43.27          & 73.00          & 92.14          & 48.50          \\
                                          & H2O                     & 20\%              & \textbf{48.52} & 28.05          & 56.97          & 52.67          & \textbf{22.17} & 53.24          & 32.14          & 29.87          & 8.96           & \textbf{99.00} & 38.78          & 24.42          & 45.88          & \textbf{43.37} & \textbf{73.00} & 91.65          & 46.79          \\
                                          & MorphKV                 & 20\%              & 48.30          & 28.41          & 57.15          & \textbf{54.63} & \textbf{23.17} & 54.14          & 30.09          & 30.35          & \textbf{10.12} & \textbf{99.50} & 38.96          & 24.26          & \textbf{47.76} & 42.51          & 52.50          & 92.36          & 45.89          \\
                                          & \textbf{Nexus Sampling} & 20\%              & 48.06          & \textbf{29.44} & \textbf{57.61} & 53.74          & 22.15          & \textbf{56.78} & \textbf{32.37} & \textbf{30.89} & 10.00          & \textbf{99.00} & \textbf{41.16} & \textbf{25.23} & 46.64          & 42.75          & 70.00          & \textbf{92.62} & \textbf{47.40} \\
            \midrule
            \multirow{4}{*}{Llama-3.2-1B} & Full Attention          & 100\%             & 31.23          & 29.65          & 35.47          & 29.83          & 25.88          & 43.35          & 18.45          & 20.40          & 3.67           & 4.50           & 16.35          & 21.84          & 36.04          & 39.77          & 61.50          & 78.54          & 31.03          \\
                                          & H2O                     & 20\%              & 30.80          & 23.38          & \textbf{35.45} & 28.62          & 21.00          & 39.66          & \textbf{17.63} & 20.65          & 2.67           & 4.50           & 14.62          & \textbf{21.16} & 34.60          & \textbf{38.92} & 61.50 & 78.50          & 29.60          \\
                                          & MorphKV                 & 20\%              & 28.89          & 23.58          & 33.64          & 28.69          & \textbf{21.51} & 34.95          & 16.78          & 19.56          & 3.17           & 4.50           & 14.50          & \textbf{21.17} & 35.12          & 36.77          & 47.00          & \textbf{79.50} & 28.08          \\
                                          & \textbf{Nexus Sampling} & 20\%              & \textbf{30.97} & \textbf{24.63} & 34.41          & \textbf{30.46} & 20.81          & \textbf{41.16} & \textbf{17.62}          & \textbf{21.59} & \textbf{4.64}  & \textbf{5.00}  & \textbf{15.18} & 21.01          & \textbf{37.48} & 38.73          & \textbf{62.00}          & 78.88          & \textbf{30.29} \\
            \midrule
            \multirow{4}{*}{Qwen3-8B}     & Full Attention          & 100\%             & 42.62          & 33.58          & 57.94          & 57.39          & 24.82          & 52.94          & 34.33          & 27.64          & 4.50           & 100.00         & 47.91          & 23.85          & 56.67          & 44.17          & 71.50          & 90.71          & 48.16          \\
                                          & H2O                     & 20\%              & 42.08          & 29.60          & 57.83          & 55.19          & 19.50          & 51.76          & 33.73          & \textbf{28.18} & 3.50           & \textbf{100.00}& 43.72          & 23.28          & 55.10          & \textbf{44.81} & \textbf{71.50} & 90.71          & 46.91          \\
                                          & MorphKV                 & 20\%              & 40.15          & 29.44          & 55.96          & \textbf{57.36} & \textbf{20.88} & 50.66          & 31.51          & 27.26          & 4.50           & \textbf{100.00}& 40.45          & 23.01          & \textbf{56.32} & 43.64          & 53.50          & 90.80          & 45.34          \\
                                          & \textbf{Nexus Sampling} & 20\%              & \textbf{42.38} & \textbf{30.80} & \textbf{59.21} & 56.79          & 20.60          & \textbf{52.17} & \textbf{34.11} & 26.96          & \textbf{5.00}  & \textbf{100.00}& \textbf{43.95} & \textbf{23.31} & 56.29          & 43.10          & 70.00          & \textbf{91.36} & \textbf{47.25} \\
            \bottomrule
        \end{tabular}%
    }
\end{table*}

Table~\ref{tab:longbench} reports per-task LongBench accuracy at 80\% eviction.
In the \emph{prefill-only} regime, Nexus Sampling matches or exceeds SnapKV and PyramidKV on the average across all three models, recovering most of the gap to Full Attention.
In the \emph{prefill\,+\,decode} regime, where eviction errors compound across decoding steps and the reservoir's product-over-steps guarantee (Lemma~\ref{lem:streaming}) is structurally relevant, Nexus Sampling wins the average against baselines.
The gain on MultiFieldQA-en (MFQA) is especially indicative of the marginal-erosion failure mode: in the prefill + decode regime Nexus Sampling tops both deterministic top-$K$ baselines (Llama-3.1-8B: 56.78 vs.\ H2O 53.24 and MorphKV 54.14) and matches Full Attention (56.16). MFQA answers often hinge on a span the recent query does not strongly attend to, a subtly important token that deterministic top-$K$ drops once its score dips below the cutoff and that reservoir sampling instead keeps alive at a controlled rate.

\subsection{RULER Across Context Lengths}
\label{sec:exp-ruler}

\begin{table}[!t]
    \centering
    \caption{RULER accuracy (\%) across 4K--64K context lengths at 80\% KV cache eviction.
        \emph{Prefill} blocks evict only during prefill; \emph{Prefill\,+\,Decode} blocks evict throughout the entire inference stream.
        \textbf{Bold} marks the leading eviction method per row (Full Attention shown for reference).
        Qwen3-8B supports a maximum context length of 32K, so its 64K column cannot be evaluated and is reported as ``--''.}
    \label{tab:ruler}
    \setlength{\tabcolsep}{3pt}
    \resizebox{0.7\linewidth}{!}{%
        \begin{tabular}{ll c | ccccc | c}
            \toprule
            \textbf{Model}                & \textbf{Method}         & \textbf{Density} & \textbf{4K}    & \textbf{8K}    & \textbf{16K}   & \textbf{32K}   & \textbf{64K}   & \textbf{Avg.}  \\
            \midrule
            \multicolumn{9}{l}{\emph{Prefill-only eviction.}}                                                                                                                                \\
            \midrule
            \multirow{4}{*}{Llama-3.1-8B} & Full Attention          & 100\%            & 96.15          & 95.91          & 95.43          & 91.35          & 86.30          & 93.03          \\
                                          & SnapKV                  & 20\%             & 82.69          & 68.51          & 73.8          & 64.42          & 73.32          & 72.55          \\
                                          & PyramidKV               & 20\%             & 75.00          & 83.65          & 86.30          & 86.30          & 81.97          & 82.64          \\
                                          & \textbf{Nexus Sampling} & 20\%             & \textbf{89.42} & \textbf{88.46} & \textbf{94.23} & \textbf{88.46} & \textbf{89.18} & \textbf{89.95} \\
            \midrule
            \multirow{4}{*}{Llama-3.2-1B} & Full Attention          & 100\%            & 76.68          & 70.67          & 65.38          & 64.66          & 60.34          & 67.55          \\
                                          & SnapKV                  & 20\%             & 45.43          & 42.55          & 37.02          & 37.02          & 30.05          & 38.41          \\
                                          & PyramidKV               & 20\%             & 41.11          & 45.91          & 43.51          & 46.88          & 41.59          & 43.80          \\
                                          & \textbf{Nexus Sampling} & 20\%             & \textbf{58.89} & \textbf{61.78} & \textbf{56.97} & \textbf{63.70} & \textbf{61.54} & \textbf{60.58} \\
            \midrule
            \multirow{4}{*}{Qwen3-8B}     & Full Attention          & 100\%            & 97.60          & 95.43          & 88.46          & 92.79          & --             & 93.57          \\
                                          & SnapKV                  & 20\%             & 70.49          & 67.01          & 73.61          & 79.17          & --             & 72.57          \\
                                          & PyramidKV               & 20\%             & 61.54          & 61.06          & 65.14          & 65.38          & --             & 63.28          \\
                                          & \textbf{Nexus Sampling} & 20\%             & \textbf{77.16} & \textbf{83.41} & \textbf{80.29} & \textbf{82.45} & --             & \textbf{80.83} \\
            \midrule
            \multicolumn{9}{l}{\emph{Prefill\,+\,decode eviction.}}                                                                                                                          \\
            \midrule
            \multirow{4}{*}{Llama-3.1-8B} & Full Attention          & 100\%            & 96.15          & 95.91          & 95.43          & 91.35          & 86.30          & 93.03          \\
                                          & H2O                     & 20\%             & 68.03          & 68.99 & 77.40          & 74.52          & 69.23          & 71.63          \\
                                          & MorphKV                 & 20\%             & 84.62          & 65.38          & 66.83          & 61.30          & 68.27          & 69.28          \\
                                          & \textbf{Nexus Sampling} & 20\%             & \textbf{91.35} & \textbf{91.35} & \textbf{94.47} & \textbf{89.18} & \textbf{84.86} & \textbf{90.24} \\
            \midrule
            \multirow{4}{*}{Llama-3.2-1B} & Full Attention          & 100\%            & 76.68          & 70.67          & 65.38          & 64.66          & 60.34          & 67.55          \\
                                          & H2O                     & 20\%             & 36.06          & 33.65          & 35.34          & 40.15          & 37.26          & 36.49          \\
                                          & MorphKV                 & 20\%             & 22.84          & 17.31          & 15.39          & 18.75          & 12.98          & 17.45          \\
                                          & \textbf{Nexus Sampling} & 20\%             & \textbf{58.89} & \textbf{59.13} & \textbf{59.14} & \textbf{62.98} & \textbf{58.89} & \textbf{59.81} \\
            \midrule
            \multirow{4}{*}{Qwen3-8B}     & Full Attention          & 100\%            & 97.60          & 95.43          & 88.46          & 92.79          & --             & 93.57          \\
                                          & H2O                     & 20\%             & 66.83          & 71.88          & 74.04          & 70.19          & --             & 70.74          \\
                                          & MorphKV                 & 20\%             & 67.31          & 62.50          & 61.30          & 57.45          & --             & 62.14          \\
                                          & \textbf{Nexus Sampling} & 20\%             & \textbf{73.56} & \textbf{81.49} & \textbf{77.16} & \textbf{84.14} & --             & \textbf{79.09} \\
            \bottomrule
        \end{tabular}%
    }
\end{table}

Table~\ref{tab:ruler} reports RULER accuracy across 4K--64K context lengths.
Note that Qwen3-8B supports a maximum context length of 32K, so we cannot evaluate it at 64K and its average is taken over the 4K--32K lengths.
Nexus Sampling is the leading eviction method in nearly every cell, with margins that grow with context length.
In \emph{prefill-only}, Nexus matches Full Attention within $\sim$3 points on Llama-3.1-8B on average, while SnapKV loses $15$--$25$ points; on the smaller Llama-3.2-1B the gap to SnapKV widens further ($60.58$ vs.\ $38.41$ on average).
In \emph{prefill\,+\,decode}, the same pattern is amplified: H2O and MorphKV both sit $\sim$20 points below Nexus at every length (avg 71.63 and 69.28 vs. Nexus 90.24), with MorphKV degrading further as the stream lengthens, while Nexus Sampling stays within $\sim$1.5 points of Full Attention on Llama-3.1-8B at 64K, and on Llama-3.2-1B, Nexus almost doubles the next-best baseline ($59.81$ vs.\ H2O $36.49$).
This is the empirical signature of the min-vs-product survival argument: on retrieval-heavy synthetic tasks where every step's decisions matter, deterministic top-$K$ silently erodes the relevant tokens while reservoir sampling does not.

\subsection{Agentic Coding}
\label{sec:exp-agentic}

We further evaluate Nexus Sampling on an agentic coding benchmark, where the model operates as a multi-turn coding agent over a long, growing interaction trace and is scored by \emph{Resolved/Pass@1} on $50$ tasks sampled from SWE-bench~\citep{jimenez2024swebench}.
Concretely, we use DeepSWE-Preview~\citep{deepswe2025}, trained on top of Qwen3-32B within the R2E-Gym environment~\citep{jain2025r2egym}.
This setting stresses eviction in two orthogonal ways.
First, we vary the \emph{scoring scope}: in \emph{Prefill-only} the cache is compressed once during prefill, while in \emph{Prefill\,+\,Decode} eviction runs throughout the entire generation stream.
Second, and more importantly, we vary the \emph{eviction protocol}.
Under \emph{Full context every turn}, the entire history is re-prefilled at each agent turn and compression is re-applied from scratch, so the cache is only an inference-time approximation of full attention and evicted tokens can re-enter on later turns; eviction is always conditioned on the current turn's query, so it is effectively a \emph{single-turn} problem of deciding what is relevant to the question being answered \emph{right now}, with full access to the history each time.
Under \emph{True eviction}, the cache persists across turns and evicted tokens are gone permanently, making it a genuinely \emph{multi-turn} problem: the method must commit to keeping or discarding a token \emph{before} knowing which later turn will need it, so it must retain information whose relevance only surfaces several turns later.
Mistakes are irreversible and compound over the agent's lifetime.

\begin{table}[ht]
    \centering
    \small
    \caption{Agentic coding accuracy (\emph{Resolved/Pass@1} over $50$ tasks) at $20\%$ KV density (80\% eviction).
        \emph{Full context every turn} re-prefills and re-compresses the entire history at every agent turn (evicted tokens can return); \emph{True eviction} carries a persistent cache across turns (evicted tokens are gone for good).
        \textbf{Bold} marks the leading eviction method per block (Dense shown for reference).}
    \label{tab:agentic}
    \setlength{\tabcolsep}{6pt}
    \begin{tabular}{l l l c c}
        \toprule
        \textbf{Scope}                                   & \textbf{Protocol}                       & \textbf{Method}         & \textbf{KV density} & \textbf{Resolved/Pass@1} \\
        \midrule
        \multirow{10}{*}{\rotatebox[origin=c]{90}{\textbf{Prefill-only}}}
                                                         & \multirow{6}{*}{\makecell[l]{Full context\\every turn}} & Dense                   & 100\%               & 12 (24\%)                \\
                                                         &                                         & SnapKV                  & 20\%                & 8 (16\%)                 \\
                                                         &                                         & PyramidKV               & 20\%                & 7 (14\%)                 \\
                                                         &                                         & AdaSnapKV               & 20\%                & 6 (12\%)                 \\
                                                         &                                         & AdaPyramidKV            & 20\%                & 6 (12\%)                 \\
                                                         &                                         & \textbf{Nexus Sampling} & 20\%                & \textbf{9 (18\%)}        \\
        \cmidrule(l){2-5}
                                                         & \multirow{6}{*}{\makecell[l]{True\\eviction}}           & Dense                   & 100\%               & 12 (24\%)                \\
                                                         &                                         & SnapKV                  & 20\%                & \textbf{8 (16\%)}        \\
                                                         &                                         & PyramidKV               & 20\%                & 1 (2\%)                  \\
                                                         &                                         & AdaSnapKV               & 20\%                & \textbf{8 (16\%)}        \\
                                                         &                                         & AdaPyramidKV            & 20\%                & 1 (2\%)                  \\
                                                         &                                         & \textbf{Nexus Sampling} & 20\%                & \textbf{8 (16\%)}        \\
        \midrule
        \multirow{8}{*}{\rotatebox[origin=c]{90}{\textbf{Prefill\,+\,Decode}}}
                                                         & \multirow{4}{*}{\makecell[l]{Full context\\every turn}} & Dense                   & 100\%               & 12 (24\%)                \\
                                                         &                                         & H2O                     & 20\%                & 1 (2\%)                  \\
                                                         &                                         & MorphKV                 & 20\%                & 6 (12\%)                 \\
                                                         &                                         & \textbf{Nexus Sampling} & 20\%                & \textbf{7 (14\%)}        \\
        \cmidrule(l){2-5}
                                                         & \multirow{4}{*}{\makecell[l]{True\\eviction}}           & Dense                   & 100\%               & 12 (24\%)                \\
                                                         &                                         & H2O                     & 20\%                & 4 (8\%)                  \\
                                                         &                                         & MorphKV                 & 20\%                & \textbf{9 (18\%)}        \\
                                                         &                                         & \textbf{Nexus Sampling} & 20\%                & 8 (16\%)                 \\
        \bottomrule
    \end{tabular}
\end{table}

Table~\ref{tab:agentic} reports the results.
Under \emph{Full context every turn}, Nexus Sampling leads all eviction baselines in both scoring scopes, confirming that the Nexus score is a stronger selection signal than direct attention even when no token is permanently lost.
The protocol shift to \emph{True eviction} is where the structural differences surface: PyramidKV and AdaPyramidKV collapse from $7$/$6$ to $1$/$1$ once their evicted tokens can no longer be recovered, the failure mode an irreversible deterministic top-$K$ verdict induces, whereas Nexus Sampling remains stable at $8$ and matches the best baseline.
In the hardest \emph{Prefill\,+\,Decode $\cdot$ True eviction} setting, Nexus Sampling stays competitive at 8, one task behind MorphKV and ahead of H2O.

\subsection{Ablations}
\label{sec:exp-ablations}

\begin{figure*}[t]
    \centering
    \includegraphics[width=0.85\linewidth]{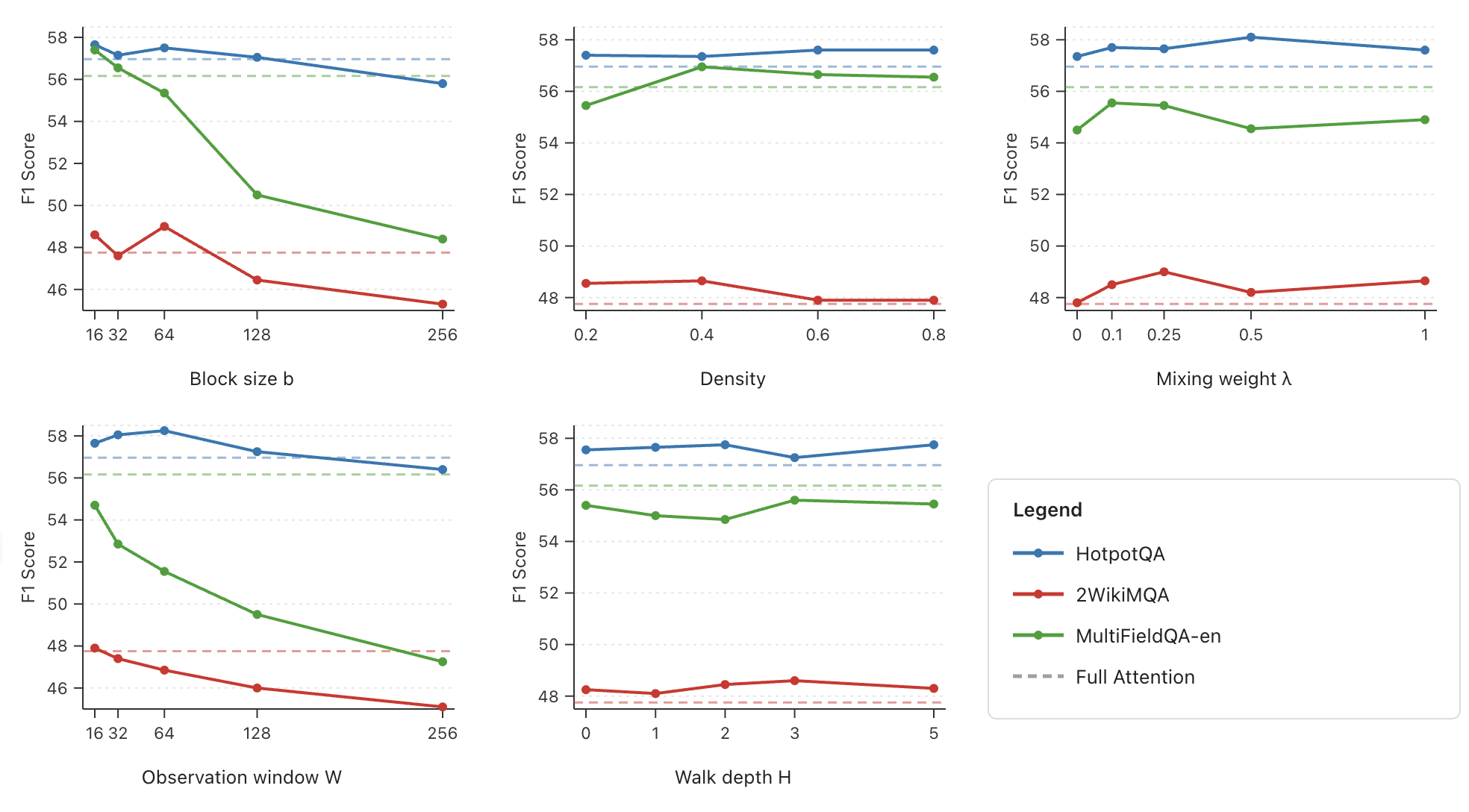}
    \caption{Ablations of the five Nexus knobs on Llama-3.1-8B at $20\%$ density, evaluated on three multi-hop QA tasks (HotpotQA, 2WikiMQA, MultiFieldQA-en).}
    \label{fig:ablations}
\end{figure*}

Nexus Sampling is robust to its hyperparameters: across walk depth $H$, mixing weight $\lambda$, observation window $W$, block size $b$, and density (Figure~\ref{fig:ablations}), the defaults of \S\ref{sec:exp-setup} sit inside a flat region on Llama-3.1-8B across three multi-hop QA tasks (HotpotQA, 2WikiMQA, MultiFieldQA-en).
First, the walk-depth sweep saturates by $H = 3$ and is indistinguishable from $H = 5$, matching the rank-1 propagation analysis of \S\ref{sec:theory-body}: a few walk steps are enough to surface bridge tokens, and additional iterations only resharpen what is already there, so $H$ does not need per-task tuning.
Second, accuracy is essentially flat from $20\%$ to $80\%$ density on all three tasks, confirming the aggressive setting we report in \S\ref{sec:exp-longbench}--\ref{sec:exp-ruler} loses no accuracy headroom; this is the empirical signature of the marginal-mass argument of \S\ref{sec:why}, where reservoir sampling preserves subtly important tokens that deterministic top-$K$ would erode rather than collapsing on the first below-cutoff step.
Details discussion appears in App.~\ref{app:ablations}.

\subsection{Inference Efficiency}
\label{sec:exp-throughput}

We measure the two efficiency axes that long-context eviction must improve under a fixed memory budget: decode throughput and steady-state per-sequence memory.
All runs use Llama-3.1-8B-Instruct in bf16 on a single H200 (143\,GB), prefill length $T_p = 8\text{K}$, batch size 1, against dense FlashAttention-2~\citep{dao2023flashattention2} and the baselines from Tables~\ref{tab:longbench}--\ref{tab:ruler}.

\begin{table}[ht]
  \centering
  \small
  \caption{Decode throughput (tokens/s) and per-step latency (ms) at $T_p = 8\text{K}$ prefill, batch size 1, on Llama-3.1-8B-Instruct (H200, bf16).
    Higher is better for throughput; lower is better for per-step latency.
    \textbf{Bold} marks the best value per row.}
  \label{tab:throughput}
  \setlength{\tabcolsep}{2pt}
  \resizebox{0.6\linewidth}{!}{%
    \begin{tabular}{c | l | c c c c}
      \toprule
      \multirow{2}{*}{\textbf{Decode $D$}} & \multirow{2}{*}{\textbf{Metric}}    & \multirow{2}{*}{\textbf{Dense}} & \multirow{2}{*}{\textbf{SnapKV}} & \multirow{2}{*}{\textbf{MorphKV}} & \textbf{Nexus}                                                  \\
                                  &                                     &                                 &                                  &                                   & \textbf{Sampling}                                               \\
      \midrule
      \multirow{2}{*}{4K}         & Throughput (tok/s, $\uparrow$)      & 35.1                            & 40.2                             & 31.4                              & \textbf{45.5}                                                   \\
                                  & Per-step (ms, $\downarrow$)         & 28.53                           & 24.85                            & 31.84                             & \textbf{22.00}                                                  \\
      \midrule
      \multirow{2}{*}{8K}         & Throughput (tok/s, $\uparrow$)      & 35.0                            & 34.2                             & 29.2                              & \textbf{45.2}                                                   \\
                                  & Per-step (ms, $\downarrow$)         & 28.55                           & 29.27                            & 34.27                             & \textbf{22.12}                                                  \\
      \midrule
      \multirow{2}{*}{16K}        & Throughput (tok/s, $\uparrow$)      & 36.0                            & 35.0                             & 32.8                              & \textbf{43.8}                                                   \\
                                  & Per-step (ms, $\downarrow$)         & 27.81                           & 28.55                            & 30.49                             & \textbf{22.83}                                                  \\
      \bottomrule
    \end{tabular}%
  }
\end{table}

\noindent\textbf{Decode throughput.}
Table~\ref{tab:throughput} reports decode-time throughput and per-step latency across decode lengths $D \in \{4\text{K}, 8\text{K}, 16\text{K}\}$.
Nexus Sampling delivers $1.22$--$1.30\times$ throughput and $15$--$23\%$ lower per-step latency over dense FlashAttention-2, beating both SnapKV and MorphKV at every $D$.

\begin{table}[ht]
  \centering
  \small
  \caption{Steady-state per-sequence decode memory ($\text{mem}_{\text{post}} - \text{mem}_{\text{after\_load}}$, model weights excluded) on Llama-3.1-8B-Instruct (bf16, H200, batch size 1) at $T_p = 8\text{K}$.
    Nexus Sampling pins the cache at a constant $\approx$1638-token budget regardless of decode length, so its decode memory is essentially flat in $D$, while dense FlashAttention-2 grows linearly.
    \textbf{Bold} marks the lower memory footprint per row.}
  \label{tab:memory}
  \setlength{\tabcolsep}{6pt}
  \begin{tabular}{c | c c | c}
    \toprule
    \textbf{Decode $D$} & \textbf{Dense FA-2} & \textbf{Nexus Sampling} & \textbf{Saving} \\
    \midrule
    4K                  & 1.5\,GiB            & \textbf{0.3\,GiB}       & $5.0\times$     \\
    8K                  & 2.0\,GiB            & \textbf{0.3\,GiB}       & $6.7\times$     \\
    16K                 & 3.0\,GiB            & \textbf{0.4\,GiB}       & $7.5\times$     \\
    32K                 & 5.0\,GiB            & \textbf{0.5\,GiB}       & $10.0\times$    \\
    \bottomrule
  \end{tabular}
\end{table}

\noindent\textbf{Per-sequence memory.}
Table~\ref{tab:memory} reports the steady-state per-sequence decode memory (GPU memory each sequence holds during decoding excluding model weights, which determines how many concurrent sequences fit under a fixed GPU budget).
Nexus Sampling saves $5.0\text{--}10.0\times$ over dense FlashAttention-2, with the saving growing with $D$ because the dense cache scales as $T_p + D$ while Nexus pins the working cache at a constant $20\%$ density of $T_p$; at $D = 32\text{K}$ this $10\times$ reduction translates into a comparable factor in achievable batch size on the same GPU.

\section{Related Work}
\label{sec:related_work}

We situate Nexus Sampling within the broader landscape of KV cache compression and reservoir sampling.
KV cache eviction methods (\textbf{StreamingLLM}~\citep{xiao2024streamingllm}, \textbf{H2O}~\citep{zhang2023h2o}, \textbf{SnapKV}~\citep{li2024snapkv}, \textbf{PyramidKV}~\citep{cai2024pyramidkv}, \textbf{Ada-KV}~\citep{feng2025adakv}, and \textbf{MorphKV}~\citep{ghadia2025morphkv}) all reduce to different scoring strategies fed into the same deterministic top-$K$ selection rule, leaving them vulnerable to the irreversible-verdict failure mode discussed in \S\ref{sec:why}.
Nexus Sampling is orthogonal to KV cache \emph{selection} (sparse attention) methods such as \textbf{Quest}~\citep{tang2024quest}, \textbf{BLASST}~\citep{yuan2025blasst}, and \textbf{Sketch-and-Walk}~\citep{le2026sketchwalk}, and can be composed with them: eviction decides which tokens stay resident in the cache, and sparse attention then decides which of the surviving tokens to read at each step, so the two stack naturally as a memory-side and a compute-side compression on the same cache.
Nexus Sampling also connects to classical weighted reservoir sampling~\citep{chao1982sampling,vitter1985random,efraimidis2006weighted}.
See Appendix~\ref{sec:app_related_work} for a detailed discussion of each method and its relationship to our approach.

\section{Conclusion}
\label{sec:conclusion}

We presented \textbf{Nexus Sampling}, a training-free KV cache eviction method that pairs an iterative walk surfacing bridge tokens with weighted reservoir sampling, replacing the deterministic top-$K$ selection that every prior eviction method shares.
Theoretically, the walk provably amplifies hub tokens above any uniform window average, and the reservoir's long-run survival decays as a product over eviction steps rather than collapsing on the first below-cutoff step.
Empirically, at $80\%$ KV cache eviction, Nexus Sampling matches dense attention within $\sim$1 point on LongBench while outperforming top-$K$ baselines on retrieval-heavy long-context tasks, with up to $10\times$ smaller per-sequence cache memory than dense FlashAttention-2.

\noindent\textbf{Implications for agentic context.}
As LLM workloads shift from one-shot prompts toward long-running agents, eviction moves from a one-time cost to a continuous, lifetime-shaping decision.
We expect that in this regime the algorithmic question of \emph{what selection primitive to use} eclipses the question of \emph{what scoring function to use}: existing methods already produce reasonable per-token weights, but only a primitive that preserves marginal mass across a long stream of eviction steps, like the reservoir step at the core of Nexus Sampling, is structurally matched to a streaming-budget setting.


\newpage
\bibliographystyle{plainnat}
\bibliography{ref}

@inproceedings{bai2024longbench,
  title     = {{LongBench}: A Bilingual, Multitask Benchmark for Long Context Understanding},
  author    = {Bai, Yushi and Lv, Xin and Zhang, Jiajie and Lyu, Hongchang and Tang, Jiankai and Huang, Zhidian and Du, Zhengxiao and Liu, Xiao and Zeng, Aohan and Hou, Lei and Dong, Yuxiao and Tang, Jie and Li, Juanzi},
  booktitle = {Proceedings of the 62nd Annual Meeting of the Association for Computational Linguistics (ACL)},
  year      = {2024}
}

@inproceedings{hsieh2024ruler,
  title     = {{RULER}: What's the Real Context Size of Your Long-Context Language Models?},
  author    = {Hsieh, Cheng-Ping and Sun, Simeng and Kriman, Samuel and Acharya, Shantanu and Rekesh, Dima and Jia, Fei and Zhang, Yang and Ginsburg, Boris},
  booktitle = {Conference on Language Modeling (COLM)},
  year      = {2024}
}

@article{jain2025r2egym,
  title   = {{R2E-Gym}: Procedural Environments and Hybrid Verifiers for Scaling Open-Weights {SWE} Agents},
  author  = {Jain, Naman and Singh, Jaskirat and Shetty, Manish and Zheng, Liang and Sen, Koushik and Stoica, Ion},
  journal = {arXiv preprint arXiv:2504.07164},
  year    = {2025}
}

@misc{deepswe2025,
  title        = {{DeepSWE}: Training a Fully Open-sourced, State-of-the-Art Coding Agent by Scaling RL},
  author       = {Luo, Michael and Jain, Naman and Singh, Jaskirat and Tan, Sijun and Patel, Ameen and Wu, Qingyang and Ariyak, Erran and Cai, Colin and Cuadron, Alejandro and Zhang, Tianjun and Stoica, Ion and Sen, Koushik},
  year         = {2025},
  howpublished = {Notion Blog},
  note         = {Agentica and Together AI}
}

@inproceedings{zhang2023h2o,
  title     = {{H2O}: Heavy-Hitter Oracle for Efficient Generative Inference of Large Language Models},
  author    = {Zhang, Zhenyu and Sheng, Ying and Zhou, Tianyi and Chen, Tianlong and Zheng, Lianmin and Cai, Ruisi and Song, Zhao and Tian, Yuandong and R{\'e}, Christopher and Barrett, Clark and Wang, Zhangyang and Chen, Beidi},
  booktitle = {Advances in Neural Information Processing Systems (NeurIPS)},
  year      = {2023}
}

@inproceedings{tang2024quest,
  title     = {{Quest}: Query-Aware Sparsity for Efficient Long-Context {LLM} Inference},
  author    = {Tang, Jiaming and Zhao, Yilong and Zhu, Kan and Xiao, Guangxuan and Kasikci, Baris and Han, Song},
  booktitle = {International Conference on Machine Learning (ICML)},
  year      = {2024}
}

@inproceedings{xiao2024streamingllm,
  title     = {Efficient Streaming Language Models with Attention Sinks},
  author    = {Xiao, Guangxuan and Tian, Yuandong and Chen, Beidi and Han, Song and Lewis, Mike},
  booktitle = {International Conference on Learning Representations (ICLR)},
  year      = {2024}
}

@article{yuan2025blasst,
  title   = {{BLASST}: Dynamic {BL}ocked Attention Sparsity via Softmax Thresholding},
  author  = {Yuan, Jiayi and Shinn, Cameron and Xu, Kai and Cui, Jingze and Klimiashvili, George and Xiao, Guangxuan and Zheng, Perkz and Li, Bo and Zhou, Yuxin and Ye, Zhouhai and You, Weijie and Zheng, Tian and Brown, Dominic and Wang, Pengbo and Hoehnerbach, Markus and Cai, Richard and Demouth, Julien and Owens, John D. and Hu, Xia and Han, Song and Liu, Timmy and Mao, Huizi},
  journal = {arXiv preprint arXiv:2512.12087},
  year    = {2025}
}

@inproceedings{cai2024pyramidkv,
  title     = {{PyramidKV}: Dynamic {KV} Cache Compression Based on Pyramidal Information Funneling},
  author    = {Cai, Zefan and Zhang, Yichi and Gao, Bofei and Liu, Yuliang and Li, Yucheng and Liu, Tianyu and Lu, Keming and Xiong, Wayne and Dong, Yue and Hu, Junjie and Xiao, Wen},
  booktitle = {Conference on Language Modeling (COLM)},
  year      = {2025}
}

@inproceedings{feng2025adakv,
  title     = {{Ada-KV}: Optimizing {KV} Cache Eviction by Adaptive Budget Allocation for Efficient {LLM} Inference},
  author    = {Feng, Yuan and Lv, Junlin and Cao, Yukun and Xie, Xike and Zhou, S. Kevin},
  booktitle = {Advances in Neural Information Processing Systems (NeurIPS)},
  year      = {2025}
}

@inproceedings{li2024snapkv,
  title     = {{SnapKV}: {LLM} Knows What You are Looking for Before Generation},
  author    = {Li, Yuhong and Huang, Yingbing and Yang, Bowen and Venkitesh, Bharat and Locatelli, Acyr and Ye, Hanchen and Cai, Tianle and Lewis, Patrick and Chen, Deming},
  booktitle = {Advances in Neural Information Processing Systems (NeurIPS)},
  year      = {2024}
}

@inproceedings{ghadia2025morphkv,
  title     = {Dialogue Without Limits: Constant-Sized {KV} Caches for Extended Responses in {LLM}s},
  author    = {Ghadia, Ravi and Kumar, Avinash and Jain, Gaurav and Nair, Prashant and Das, Poulami},
  booktitle = {International Conference on Machine Learning (ICML)},
  year      = {2025}
}

@article{vitter1985random,
  title   = {Random Sampling with a Reservoir},
  author  = {Vitter, Jeffrey S.},
  journal = {ACM Transactions on Mathematical Software},
  volume  = {11},
  number  = {1},
  pages   = {37--57},
  year    = {1985}
}

@article{efraimidis2006weighted,
  title   = {Weighted Random Sampling with a Reservoir},
  author  = {Efraimidis, Pavlos S. and Spirakis, Paul G.},
  journal = {Information Processing Letters},
  volume  = {97},
  number  = {5},
  pages   = {181--185},
  year    = {2006}
}

@article{chao1982sampling,
  title   = {A General Purpose Unequal Probability Sampling Plan},
  author  = {Chao, M.-T.},
  journal = {Biometrika},
  volume  = {69},
  number  = {3},
  pages   = {653--656},
  year    = {1982}
}

@inproceedings{dao2022flashattention,
  title     = {{FlashAttention}: Fast and Memory-Efficient Exact Attention with {IO}-Awareness},
  author    = {Dao, Tri and Fu, Daniel Y. and Ermon, Stefano and Rudra, Atri and R{\'e}, Christopher},
  booktitle = {Advances in Neural Information Processing Systems (NeurIPS)},
  year      = {2022}
}

@inproceedings{dao2023flashattention2,
  title     = {{FlashAttention-2}: Faster Attention with Better Parallelism and Work Partitioning},
  author    = {Dao, Tri},
  booktitle = {International Conference on Learning Representations (ICLR)},
  year      = {2024}
}

@inproceedings{pope2023efficiently,
  title     = {Efficiently Scaling Transformer Inference},
  author    = {Pope, Reiner and Douglas, Sholto and Chowdhery, Aakanksha and Devlin, Jacob and Bradbury, James and Levskaya, Anselm and Heek, Jonathan and Xiao, Kefan and Agrawal, Shivani and Dean, Jeff},
  booktitle = {Proceedings of Machine Learning and Systems (MLSys)},
  year      = {2023}
}

@misc{openai2024o1,
  title        = {{OpenAI} o1 System Card},
  author       = {{OpenAI}},
  year         = {2024},
  howpublished = {Technical report, OpenAI}
}

@misc{anthropic2025sonnet,
  title        = {{Claude Sonnet 4.5}},
  author       = {{Anthropic}},
  year         = {2025},
  howpublished = {Anthropic model release}
}

@inproceedings{jimenez2024swebench,
  title     = {{SWE}-bench: Can Language Models Resolve Real-World {GitHub} Issues?},
  author    = {Jimenez, Carlos E. and Yang, John and Wettig, Alexander and Yao, Shunyu and Pei, Kexin and Press, Ofir and Narasimhan, Karthik},
  booktitle = {International Conference on Learning Representations (ICLR)},
  year      = {2024}
}

@article{le2026sketchwalk,
  title   = {Scout Before You Attend: Sketch-and-Walk Sparse Attention for Efficient {LLM} Inference},
  author  = {Le, Hoang Anh Duy and Joshi, Sahil and Yang, Zeyu and Xu, Zhaozhuo and Shrivastava, Anshumali},
  journal = {arXiv preprint arXiv:2602.07397},
  year    = {2026}
}

@article{le2026fafo,
  title   = {{FAFO}: Lossy {KV} Cache Compression for Lossless Inference Acceleration via Draftless Fumble Decoding},
  author  = {Le, Hoang Anh Duy and Zhong, Shaochen and Lu, Yifan and Dou, Yingtong and Yuan, Jiayi and Chuang, Yu-Neng and Fan, Xiran and Wang, Guanchu and Chen, Yuzhong and Hu, Xia},
  journal = {OpenReview preprint},
  note    = {\url{https://openreview.net/forum?id=oSk9tP5Mgs}},
  year    = {2026}
}

@article{joshi2026socket,
  title   = {{SOCKET}: {SO}ft Collision Kernel {E}sTimator for Sparse Attention},
  author  = {Joshi, Sahil and Chowdhury, Agniva and Bellinger, Wyatt and Kanakamedala, Amar and Singh, Ekam and Le, Hoang Anh Duy and Desai, Aditya and Shrivastava, Anshumali},
  journal = {arXiv preprint arXiv:2602.06283},
  year    = {2026}
}

@inproceedings{yuan2024longctx,
  title     = {{KV} Cache Compression, But What Must We Give in Return? {A} Comprehensive Benchmark of Long Context Capable Approaches},
  author    = {Yuan, Jiayi and Liu, Hongyi and Zhong, Shaochen and Chuang, Yu-Neng and Li, Songchen and Wang, Guanchu and Le, Duy and Jin, Hongye and Chaudhary, Vipin and Xu, Zhaozhuo and Liu, Zirui and Hu, Xia},
  booktitle = {Findings of the Association for Computational Linguistics: EMNLP 2024},
  pages     = {4623--4648},
  year      = {2024}
}

\newpage
\appendix

\section{Theoretical Analysis}
\label{sec:theory}

This appendix gives full proofs for the three statements of \S\ref{sec:theory-body} (Lemma~\ref{lem:streaming}, Theorem~\ref{thm:main}, and Proposition~\ref{prop:eviction-quality}), together with several supporting statements that we develop only here: reservoir unbiasedness (Lemma~\ref{lem:ares}), $n$-averaged concentration (Lemma~\ref{lem:nares}), window noise reduction (Lemma~\ref{lem:window}), and the walk decomposition that underlies Theorem~\ref{thm:main} (Lemma~\ref{lem:walk-decomp}).
Main-text statements are restated verbatim (or in a slightly expanded form when convenient for the proof); appendix-only statements are presented in full.

\subsection{Notation}
\label{app:notation}

\begin{table}[h]
    \centering
    \small
    \begin{tabularx}{\columnwidth}{lX}
        \toprule
        Symbol                                                       & Meaning                                                                                       \\
        \midrule
        $W$                                                          & observation window length (recent query tokens)                                               \\
        $\mathbf{Q} \in \mathbb{R}^{W \times D}$                     & trailing $W$ query rows, head-averaged                                                        \\
        $\mathbf{K} \in \mathbb{R}^{T_k \times D}$                   & full key cache, head-averaged                                                                 \\
        $b$                                                          & tokens per key block                                                                          \\
        $N_k = \lceil T_k / b \rceil$                                & number of key blocks, indexed by $j$                                                          \\
        $\mathbf{P} \in \mathbb{R}^{W \times T_k}$                   & token-level attention matrix over the observation window                                      \\
        $\mathbf{M} \in \mathbb{R}^{W \times N_k}$                   & block-aggregated attention; $M_{w,j} = \sum_{t=(j-1)b+1}^{jb} P_{w,t}$                        \\
        $\hat{\mathbf{S}} \in \mathbb{R}^{W \times N_k}$             & row-normalized block attention; $\hat{S}_{w,j} = M_{w,j}/\sum_{j'} M_{w,j'}$                  \\
        $\mathbf{a}^{(q)} = \hat{\mathbf{S}}_{q,\cdot}$              & attention distribution of query row $q$                                                       \\
        $\mathbf{a}_\star = \mathbb{E}[\mathbf{a}^{(q)}]$            & population mean block-attention distribution                                                  \\
        $\boldsymbol{\Sigma}$                                        & covariance matrix; $\Sigma_{jk} = \Cov(\hat{S}_{q,j}, \hat{S}_{q,k})$                         \\
        $\mathbf{C} \in \mathbb{R}^{N_k}$                            & hub score accumulator from the multi-hop walk                                                 \\
        $\gamma_q = 1 + \langle \mathbf{a}^{(q)}, \mathbf{C}\rangle$ & per-step amplification factor                                                                 \\
        $h_j = \boldsymbol{\Sigma}_j^\top \mathbf{a}_\star$          & hub score of block $j$; $\boldsymbol{\Sigma}_j$ is the $j$-th column of $\boldsymbol{\Sigma}$ \\
        $w_j$                                                        & combined sampling weight (Eq.~\ref{eq:combined})                                              \\
        $\pi_j$                                                      & reservoir priority of block $j$ (Eq.~\ref{eq:ares})                                           \\
        $n$                                                          & reservoir averaging count                                                                     \\
        $K$                                                          & block retention budget after forced blocks are removed                                        \\
        $e$                                                          & cache-update event index in the streaming analysis                                            \\
        \bottomrule
    \end{tabularx}
    \caption{Notation used in \S\ref{sec:method} and Appendix~\ref{sec:theory}.}
    \label{tab:notation}
\end{table}

\subsection{Reservoir Sampling: Unbiasedness and Concentration}
\label{app:ares}

\begin{lemma}[Reservoir unbiasedness, $n=1$; \citealp{efraimidis2006weighted}]
    \label{lem:ares}
    Let $\pi_j = u_j^{1/w_j}$ with $u_j \sim \mathcal{U}(0,1)$ independent, and let $I_j=\mathbf{1}\{j\in\mathcal{S}\}$ indicate whether block $j$ belongs to the top-$K$ priority set.
    The selected set $\mathcal{S}$ is a probability-proportional-to-size sample without replacement.
    In particular $p_j=\Pr(I_j=1)>0$ whenever $w_j>0$, $p_j$ is monotone in $w_j$, and for any deterministic utility vector $\mathbf{z}$,
    \[
        \widehat{Z}_{\mathrm{HT}}
        =
        \sum_{j=1}^{N_k}\frac{I_j z_j}{p_j}
        \quad\text{satisfies}\quad
        \mathbb{E}[\widehat{Z}_{\mathrm{HT}}]=\sum_{j=1}^{N_k}z_j .
    \]
\end{lemma}

\begin{proof}
    The priority rule of \citet{efraimidis2006weighted} is equivalent to drawing independent exponential clocks $E_j=-\log(u_j)/w_j$ with rates $w_j$ and taking the first $K$ arrivals. This is the standard probability-proportional-to-size without-replacement law. Positivity follows because a positive-rate exponential clock can arrive among the first $K$ with non-zero probability; monotonicity follows by coupling two clocks with the same $u_j$, since $u_j^{1/w_j}$ increases with $w_j$ (equivalently, $-\log(u_j)/w_j$ decreases with $w_j$). The Horvitz--Thompson identity is immediate:
    $\mathbb{E}[I_j z_j/p_j]=z_j$ for each $j$, and summing over blocks gives the claim.
\end{proof}

The statement above is intentionally phrased in terms of inclusion probabilities $p_j$ rather than the simplified formula $K w_j/\sum_{j'}w_{j'}$.
For fixed-size weighted sampling without replacement, the exact marginal $p_j$ has no universal linear closed form for all weights and budgets; what Nexus needs is the weaker but correct fact that positive-weight blocks have positive, monotone inclusion probability and admit unbiased retained-mass estimation.

\begin{lemma}[$n$-averaged concentration]
    \label{lem:nares}
    Let $\pi_j^{(n)} = \tfrac{1}{n}\sum_{i=1}^n (u_j^{(i)})^{1/w_j}$ with $u_j^{(i)} \sim \mathcal{U}(0,1)$ independent.
    Then $\mathbb{E}[\pi_j^{(n)}] = w_j / (w_j + 1)$ and $\mathrm{Var}(\pi_j^{(n)}) \to 0$ at rate $1/n$ as $n \to \infty$.
    If the weights $w_j$ are distinct (as guaranteed in practice by the recency tie-break $\varepsilon_{\mathrm{tie}}\,r_j$), the rank of block $j$ under $\pi_j^{(n)}$ converges almost surely to the rank under $w_j / (w_j + 1)$ as $n \to \infty$.
\end{lemma}
\begin{proof}
    For $u \sim \mathcal{U}(0,1)$, the random variable $X = u^{1/w}$ has CDF $F_X(x) = \Pr(u \leq x^w) = x^w$ on $[0,1]$, density $f_X(x) = w x^{w-1}$, and moments $\mathbb{E}[X^k] = w/(w+k)$. Hence $\mathbb{E}[X] = w/(w+1)$ and $\mathrm{Var}(X) = w/(w+2) - (w/(w+1))^2 = w/[(w+1)^2(w+2)]$. Averaging $n$ i.i.d.\ copies gives $\mathbb{E}[\pi_j^{(n)}] = w_j/(w_j+1)$ and $\mathrm{Var}(\pi_j^{(n)}) = w_j/[n(w_j+1)^2(w_j+2)] = O(1/n)$. By the strong law of large numbers, $\pi_j^{(n)} \xrightarrow{a.s.} w_j/(w_j+1)$, and the monotonicity of $w/(w+1)$ in $w$ ensures that the rank under $\pi_j^{(n)}$ converges a.s.\ to the rank under $w_j/(w_j+1)$.
\end{proof}

Because $w_j/(w_j+1)$ is strictly increasing in $w_j$, the deterministic limit induces the same ordering as top-$K$ by $w_j$ directly.
The averaging count $n$ is therefore a \emph{single-knob} interpolation between unbiased sampling (Lemma~\ref{lem:ares}) and deterministic top-$K$ ($n \to \infty$), with no other change to the pipeline.

\subsection{Streaming Marginal-Mass Survival}
\label{app:streaming}

\noindent\textbf{Lemma~\ref{lem:streaming} (Min over steps vs.\ product over steps; restated).}
\emph{Consider a sequence of eviction steps $e = 1,\dots,E$ with per-step weights $w_j^{(e)} > 0$ and a fixed retention budget $K$ at every step.
    Let $S_j^{\text{top-}K}(E)$ and $S_j^{\text{res}}(E)$ denote the long-run survival probability of block $j$ under deterministic top-$K$ and reservoir sampling respectively.
    Then $S_j^{\text{top-}K}(E) = \prod_{e=1}^E \mathbf{1}\{w_j^{(e)} \geq w_{(K)}^{(e)}\}$, which collapses to zero as soon as any single step has $w_j^{(e)} < w_{(K)}^{(e)}$; while $S_j^{\text{res}}(E) = \prod_{e=1}^E q_j^{(e)}$ with $q_j^{(e)} \in (0,1]$ strictly positive whenever $w_j^{(e)} > 0$.}
\begin{proof}
    Under deterministic top-$K$, block $j$ survives step $e$ iff $w_j^{(e)} \geq w_{(K)}^{(e)}$.
    Survival of all $E$ steps is therefore the indicator $\prod_{e=1}^E \mathbf{1}\{w_j^{(e)} \geq w_{(K)}^{(e)}\}$, which is zero whenever a single step falls below the cutoff.
    Under reservoir sampling, Lemma~\ref{lem:ares} gives a positive per-step inclusion probability $q_j^{(e)}$ whenever $w_j^{(e)} > 0$.
    With independent priority draws at different eviction steps, the survival probability is $\prod_e q_j^{(e)}$, which is strictly positive whenever every $q_j^{(e)}$ is.
\end{proof}

This is the precise sense in which the reservoir step is the \emph{right} primitive for the streaming-budget regime: deterministic top-$K$ converts a single marginal-rank misjudgment into permanent, irrecoverable loss (the block is gone from the cache and cannot re-enter), while reservoir sampling assigns strictly positive inclusion probability at every event, so the long-run survival probability remains bounded away from zero as long as the block's per-event weight is positive.

\subsection{Window Averaging as Noise Reduction}
\label{app:window}

\begin{lemma}[Window noise reduction]
    \label{lem:window}
    If each row $\mathbf{a}^{(q)} = \boldsymbol{\mu} + \boldsymbol{\xi}_q$ with $\mathbb{E}[\boldsymbol{\xi}_q] = \mathbf{0}$ and per-coordinate variance $\leq \nu^2$, then with probability $\geq 1 - \delta$,
    \[
        \|\mathbf{a} - \boldsymbol{\mu}\|_\infty \leq \nu\sqrt{2\log(2N_k/\delta)/W}.
    \]
\end{lemma}
\begin{proof}
    By construction $\mathbf{a} - \boldsymbol{\mu} = (1/W)\sum_q \boldsymbol{\xi}_q$, an average of $W$ independent zero-mean coordinate-wise sub-Gaussian random vectors. For any block $j$, Hoeffding's inequality (applied to bounded coordinates of $\boldsymbol{\xi}_q$ on $[-1,1]$) gives
    \[
        \Pr\!\bigl(|\,(\mathbf{a}-\boldsymbol{\mu})_j\,| \geq t\bigr) \leq 2\exp(-Wt^2/(2\nu^2)).
    \]
    Taking a union bound over the $N_k$ block coordinates and solving for $t$ at confidence $1-\delta$ yields the claim.
\end{proof}

The bound is the standard $1/\sqrt{W}$ noise-reduction rate, with the $N_k$-dimensional union bound contributing only a $\log N_k$ factor.
The implication for the reservoir step is that a modest window length ($W \approx 16$--$32$) suffices to bring the per-block noise floor on the reservoir's input weights below the typical scale of meaningful per-block weight differences.

\subsection{Eviction Quality Under Approximate Future Utility}
\label{app:eviction-quality}

We restate Proposition~\ref{prop:eviction-quality} of \S\ref{sec:theory-body} in its full three-part form (the main-text statement combines parts 1 and 2 into a single bound) and prove it.
Let $z_j \geq 0$ denote the utility that block $j$ would contribute to future attention if it remained in cache; this can be instantiated as future block-attention mass, or as a value-norm-weighted utility $z_j=\sum_{\tau} \alpha_{\tau j}\|V_j\|$ over a short future horizon.

\noindent\textbf{Proposition~\ref{prop:eviction-quality} (restated and extended).}
\emph{Fix one cache-update event and let $I_j$ be the reservoir-sampling inclusion indicator with inclusion probability $p_j=\Pr(I_j=1)$. For any deterministic future-utility vector $\mathbf{z}$:}
\begin{enumerate}[leftmargin=*,topsep=2pt,itemsep=1pt]
    \item \emph{The Horvitz--Thompson retained-utility estimator}
          \[
              \widehat{Z}_{\mathrm{HT}}=\sum_{j=1}^{N_k}\frac{I_j z_j}{p_j}
          \]
          \emph{is unbiased for the full-cache utility $Z=\sum_j z_j$.}
    \item \emph{If $p_j\geq p_{\min}>0$ and $0\leq z_j\leq Z_{\max}$, then with probability at least $1-\delta$,}
          \[
              |\widehat{Z}_{\mathrm{HT}}-Z|
              \leq
              \sqrt{\frac{K}{\delta}\sum_{j=1}^{N_k}\frac{z_j^2}{p_j}}
              \leq
              Z_{\max}\sqrt{\frac{K N_k}{\delta p_{\min}}}.
          \]
    \item \emph{If the Nexus weight approximates future utility as $\| \mathbf{w}-\mathbf{z}\|_\infty\leq\eta$, then the expected evicted utility $L=\sum_j(1-I_j)z_j$ satisfies}
          \[
              \left|
              \mathbb{E}[L]
              -
              \sum_{j=1}^{N_k}(1-p_j)w_j
              \right|
              \leq
              \eta\sum_{j=1}^{N_k}(1-p_j)
              \leq
              \eta N_k .
          \]
\end{enumerate}

\begin{proof}
    The first claim is the Horvitz--Thompson identity already used in Lemma~\ref{lem:ares}. For the second claim, let $A_j=I_jz_j/p_j$. Because exactly $K$ blocks are retained, $(\sum_j A_j)^2 \leq K\sum_j A_j^2$. Therefore
    \[
        \mathbb{E}[\widehat{Z}_{\mathrm{HT}}^2]
        \leq
        K\sum_j \mathbb{E}\!\left[\frac{I_jz_j^2}{p_j^2}\right]
        =
        K\sum_j \frac{z_j^2}{p_j}.
    \]
    Since $\operatorname{Var}(\widehat{Z}_{\mathrm{HT}})\leq \mathbb{E}[\widehat{Z}_{\mathrm{HT}}^2]$, Chebyshev's inequality yields the stated bound. The final inequality follows from $z_j^2\leq Z_{\max}^2$ and $p_j\geq p_{\min}$. For the third claim,
    \[
        \mathbb{E}[L]=\sum_j(1-p_j)z_j,
    \]
    and subtracting $\sum_j(1-p_j)w_j$ gives
    $\left|\sum_j(1-p_j)(z_j-w_j)\right|\leq \eta\sum_j(1-p_j)$.
\end{proof}

The proposition separates the two sources of approximation.
The reservoir step contributes no systematic bias once inclusion probabilities are accounted for; the remaining eviction error is controlled by how well the score $\mathbf{w}$ predicts future utility.
This is where the observation window and walk matter: Lemma~\ref{lem:window} reduces estimation noise in direct utility, while Theorem~\ref{thm:main} raises the weight of hub blocks whose future utility is indirect rather than visible in one-hop attention.

\subsection{Multi-Hop Walk: Decomposition and Hub Amplification}
\label{app:walk}

This subsection states a rank-1 decomposition of the walk recurrence that underlies the hub-amplification argument, then proves Theorem~\ref{thm:main}.

\begin{lemma}[Walk decomposition and hub-score zero-sum]
    \label{lem:walk-decomp}
    Let $\mathbf{M}_q = \mathbf{a}^{(q)} (\mathbf{a}^{(q)})^\top \in \mathbb{R}^{N_k \times N_k}$ be the rank-1 per-query block-affinity matrix.
    The walk recurrence~\eqref{eq:walk} satisfies the exact identity
    \begin{align*}
        \mathbf{C}^{(H)}
        &\;=\; \Bigl(\textstyle\sum_{q=1}^H \mathbf{a}^{(q)}\Bigr) \\
        &\quad+\; \Bigl(\textstyle\sum_{q=1}^H \mathbf{M}_q\,\mathbf{C}^{(q-1)}\Bigr).
    \end{align*}
    Furthermore, $\sum_j h_j = 0$, i.e., the hub scores sum to zero.
\end{lemma}
\begin{proof}
    For each step $q$, the update increment is
    \begin{align*}
        \gamma_q\,\mathbf{a}^{(q)}
        &= \bigl(1 + \langle\mathbf{a}^{(q)},\mathbf{C}^{(q-1)}\rangle\bigr)\,\mathbf{a}^{(q)} \\
        &= \mathbf{a}^{(q)} + \bigl(\mathbf{a}^{(q)\top}\mathbf{C}^{(q-1)}\bigr)\,\mathbf{a}^{(q)} \\
        &= \mathbf{a}^{(q)} + \mathbf{M}_q\,\mathbf{C}^{(q-1)},
    \end{align*}
    an exact equality with no residual. Summing $\mathbf{C}^{(q)} = \mathbf{C}^{(q-1)} + \mathbf{a}^{(q)} + \mathbf{M}_q\,\mathbf{C}^{(q-1)}$ from $q=1$ to $H$ with $\mathbf{C}^{(0)}=\mathbf{0}$ yields the stated identity. For the hub-sum claim: since $\sum_j a_j^{(q)} = 1$ a.s., we have $\sum_j \xi_{q,j} = 0$ a.s., so $\sum_j \Sigma_{jk} = \mathbb{E}[\sum_j \xi_{q,j}\,\xi_{q,k}] = 0$ for every $k$, giving $\boldsymbol{1}^\top\boldsymbol{\Sigma} = \mathbf{0}^\top$ and $\sum_j h_j = \boldsymbol{1}^\top\boldsymbol{\Sigma}\,\mathbf{a}_\star = 0$.
\end{proof}

The first term is the unwalked window aggregate (essentially $W \mathbf{a}$ if all window rows are used).
The second term is the genuinely multi-hop contribution: each $\mathbf{M}_q$ is a rank-1 affinity matrix encoding which block pairs query $q$ connects, and $\mathbf{M}_q\,\mathbf{C}^{(q-1)}$ aggregates the accumulated walk under that affinity structure.
The depth $H$ controls the highest-order hop captured; as in \citet{le2026sketchwalk}, small $H$ (e.g., $H = 3$) suffices in practice because effective composed attention depth in transformer stacks is empirically shallow.

\subsubsection{Hub Amplification (Main Theorem)}

We first record the i.i.d.\ stationarity assumption and the formal definition of a hub block, then restate Theorem~\ref{thm:main} with the full intermediate quantities and prove it.

\begin{assumption}[i.i.d.\ stationarity]
    \label{asm:iid}
    The $H$ walk-step attention distributions $\mathbf{a}^{(1)},\dots,\mathbf{a}^{(H)}$ (i.e., the $H$ query rows consumed by the recurrence~\eqref{eq:walk}) are drawn i.i.d.\ from a fixed distribution with mean $\mathbf{a}_\star$ and covariance matrix $\boldsymbol{\Sigma}$.
\end{assumption}

\begin{definition}[Hub block]
    \label{def:hub}
    Block $j$ is a \emph{hub block} if $h_j = \boldsymbol{\Sigma}_j^\top \mathbf{a}_\star > 0$, i.e., queries that give high attention to block $j$ tend to also give high attention to other frequently attended blocks. Block $j$ is \emph{peripheral} if $h_j \leq 0$.
\end{definition}

\noindent\textbf{Theorem~\ref{thm:main} (restated and extended).}
\emph{Under Assumption~\ref{asm:iid}, the Nexus walk (Eq.~\ref{eq:walk}) with $H$ steps satisfies}
\begin{equation}
    \mathbb{E}[C_j^{(H)}]
    \;=\;
    B_H\,a_{\star,j}
    \;+\;
    D_H\,h_j
    \;+\;
    O(H\,\|\boldsymbol{\Sigma}\|_F^2),
    \label{eq:hub}
\end{equation}
\emph{where $B_H = \sum_{q=1}^H \mathbb{E}[\gamma_q] \geq H$ is a block-independent amplification factor, $D_H = \sum_{q=1}^{H-1} B_q \geq H(H-1)/2$ is the cumulative covariance-boost coefficient, and the $O(\cdot)$ term captures second-order covariance interactions. Since $\sum_j h_j = 0$ (Lemma~\ref{lem:walk-decomp}), $L_1$-normalization gives}
\begin{align*}
    \mathbb{E}[\tilde{c}_j^{(H)}]
    &\;=\; a_{\star,j} + \tfrac{D_H}{B_H}\,h_j + O(\|\boldsymbol{\Sigma}\|_F^2); \\
    &\text{in particular,}\quad
    \mathbb{E}[\tilde{c}_j^{(H)}] \;\gtrless\; a_{\star,j}
    \;\Longleftrightarrow\; h_j \;\gtrless\; 0.
\end{align*}
\emph{Furthermore, hub blocks have strictly higher expected sampling weight $\mathbb{E}[w_j]$; by the monotonicity of inclusion probability in $w_j$ (Lemma~\ref{lem:ares}), their expected inclusion probability $\mathbb{E}[p_j]$ is strictly higher than the window average $\mathbf{a}$ alone would predict.}

\begin{proof}
    \noindent\textbf{Step 1: Unrolling the recurrence.}
    The walk (Eq.~\ref{eq:walk}) gives $\mathbf{C}^{(H)} = \sum_{q=1}^{H} \gamma_q\,\mathbf{a}^{(q)}$, where $\gamma_q = 1 + \langle\mathbf{a}^{(q)}, \mathbf{C}^{(q-1)}\rangle$ and $\gamma_1 = 1$. Taking expectations entry-wise:
    \begin{align}
        \mathbb{E}[C_j^{(H)}]
        &= \sum_{q=1}^{H} \mathbb{E}[\gamma_q\,a_j^{(q)}] \notag \\
        &= \sum_{q=1}^{H}
        \Bigl(
        \mathbb{E}[\gamma_q]\,a_{\star,j}
        + \Cov(\gamma_q,\,a_j^{(q)})
        \Bigr).
        \label{eq:decomp}
    \end{align}

    \noindent\textbf{Step 2: Block-independence of the amplification sum.}
    Under Assumption~\ref{asm:iid}, $\mathbf{a}^{(q)}$ is independent of $\mathbf{C}^{(q-1)}$, so
    \[
        \mathbb{E}[\gamma_q] = 1 + \langle\mathbf{a}_\star,\,\mathbb{E}[\mathbf{C}^{(q-1)}]\rangle.
    \]
    This depends only on $\mathbf{a}_\star$ and $\boldsymbol{\Sigma}$ through $\mathbb{E}[\mathbf{C}^{(q-1)}]$, and crucially \emph{not} on block index $j$. Defining $B_H = \sum_{q=1}^H \mathbb{E}[\gamma_q]$, the first term in~\eqref{eq:decomp} contributes $B_H\,a_{\star,j}$ uniformly across all blocks.

    \noindent\textbf{Step 3: Covariance boost.}
    For $q \geq 2$, by the independence of $\mathbf{a}^{(q)}$ from $\mathbf{C}^{(q-1)}$ (Assumption~\ref{asm:iid}):
    \begin{align*}
        \Cov(\gamma_q,\,a_j^{(q)})
         & = \Cov\!\bigl(\langle\mathbf{a}^{(q)},\mathbf{C}^{(q-1)}\rangle,\;a_j^{(q)}\bigr) \\
         & = \mathbb{E}[\mathbf{C}^{(q-1)}]^\top\boldsymbol{\Sigma}_j,
    \end{align*}
    where $\boldsymbol{\Sigma}_j$ is the $j$-th column of $\boldsymbol{\Sigma} = \mathbb{E}[\mathbf{a}^{(q)}\mathbf{a}^{(q)\top}] - \mathbf{a}_\star\mathbf{a}_\star^\top$. (The equality follows by conditioning on $\mathbf{C}^{(q-1)}$ and using $\mathbb{E}[\mathbf{a}^{(q)}(\mathbf{a}^{(q)})^\top] = \boldsymbol{\Sigma} + \mathbf{a}_\star\mathbf{a}_\star^\top$.) We now compute $\mathbb{E}[\mathbf{C}^{(q-1)}]$ by induction. The base case is $\mathbb{E}[\mathbf{C}^{(1)}] = \mathbf{a}_\star = B_1\,\mathbf{a}_\star$ since $\gamma_1 = 1$. At each step, $\mathbb{E}[\mathbf{C}^{(q)}] = \mathbb{E}[\mathbf{C}^{(q-1)}] + \mathbb{E}[\gamma_q]\,\mathbf{a}_\star + O(\|\boldsymbol{\Sigma}\|)$, giving
    \begin{align*}
        \mathbb{E}[\mathbf{C}^{(q-1)}] &= B_{q-1}\,\mathbf{a}_\star + O(\|\boldsymbol{\Sigma}\|), \\
        B_{q-1} &= \textstyle\sum_{q'=1}^{q-1}\mathbb{E}[\gamma_{q'}].
    \end{align*}
    Substituting:
    \[
        \Cov(\gamma_q,\,a_j^{(q)}) = B_{q-1}\,h_j + O(\|\boldsymbol{\Sigma}\|_F^2).
    \]
    Summing over $q = 2,\dots,H$ and setting $D_H = \sum_{q=1}^{H-1} B_q$:
    \[
        \sum_{q=2}^{H} \Cov(\gamma_q,\,a_j^{(q)})
        = D_H\,h_j + O(H\|\boldsymbol{\Sigma}\|_F^2).
    \]
    Combining with Step~2 gives Eq.~\eqref{eq:hub}. Note $D_H \geq \sum_{q=1}^{H-1} q = H(H-1)/2$ since $B_q \geq q$.

    \noindent\textbf{Step 4: Directional consequence.}
    By Lemma~\ref{lem:walk-decomp}, $\sum_j h_j = 0$. Therefore the $L_1$-norm of $\mathbb{E}[\mathbf{C}^{(H)}]$ satisfies
    \begin{align*}
        \textstyle\sum_j \mathbb{E}[C_j^{(H)}]
        &= B_H\underbrace{\textstyle\sum_j a_{\star,j}}_{=\,1}
        + D_H\underbrace{\textstyle\sum_j h_j}_{=\,0} \\
        & + O(H\|\boldsymbol{\Sigma}\|_F^2) \\
        &= B_H + O(H\|\boldsymbol{\Sigma}\|_F^2).
    \end{align*}
    Dividing Eq.~\eqref{eq:hub} entry-wise by this norm:
    \begin{align*}
        \mathbb{E}[\tilde{c}_j^{(H)}]
        &= \frac{B_H\,a_{\star,j} + D_H\,h_j}{B_H} + O(\|\boldsymbol{\Sigma}\|_F^2) \\
        &= a_{\star,j} + \tfrac{D_H}{B_H}\,h_j + O(\|\boldsymbol{\Sigma}\|_F^2).
    \end{align*}
    Since $D_H/B_H > 0$, we have $\mathbb{E}[\tilde{c}_j^{(H)}] \gtrless a_{\star,j}$ if and only if $h_j \gtrless 0$.

    \noindent\textbf{Step 5: Hub retention.}
    By Lemma~\ref{lem:ares}, the inclusion probability $p_j(\mathbf{w}) = \Pr(I_j=1\mid\mathbf{w})$ is increasing in $w_j$ with all other weights fixed. The realized weight $w_j = C_j^{(H)}$ is random; taking expectations over $\mathbf{w}$ gives $\mathbb{E}[p_j] = \mathbb{E}_\mathbf{w}[p_j(\mathbf{w})]$. Hub blocks satisfy $\mathbb{E}[w_j] = \mathbb{E}[C_j^{(H)}] > a_{\star,j}$ after normalization (Step~4), while peripheral blocks have $\mathbb{E}[w_j] < a_{\star,j}$. Because $p_j(\mathbf{w})$ is increasing in $w_j$, a shift in the distribution of $w_j$ toward larger values (higher mean) raises $\mathbb{E}[p_j]$ via the law of total expectation. Consequently, the expected inclusion probability of hub blocks strictly exceeds what the unwalked window score $a_j$ would predict.
\end{proof}

\noindent\textbf{Remark.}
Assumption~\ref{asm:iid} requires i.i.d.\ query rows, which is an idealization: consecutive tokens within a document are correlated.
In practice the window acts as a local mixing device; the theorem should be read as characterizing the \emph{tendency} of the walk to amplify cross-query consensus blocks rather than as an exact finite-sample guarantee.

\section{Detailed Ablations}
\label{app:ablations}

This appendix expands \S\ref{sec:exp-ablations}, walking through each panel of Figure~\ref{fig:ablations}.
All settings match \S\ref{sec:exp-setup}: Llama-3.1-8B at $20\%$ density, evaluated on three multi-hop QA tasks (HotpotQA, 2WikiMQA, MultiFieldQA-en).
We selected this slice deliberately: multi-hop QA stresses the two failure modes Nexus Sampling is designed to address (bridge tokens and irreversible eviction), so any knob that mattered should move the numbers here before showing up on broader benchmarks.

\noindent\textbf{Walk depth $H$.}
$H = 0$ (no walk) is already strong because the windowed direct score $\mathbf{a}$ alone is informative; this is the regime every direct-attention baseline already operates in, which is why those baselines are competitive on retrieval tasks where a single bridge hop suffices.
$H \geq 1$ adds a small but consistent improvement on 2WikiMQA, the most multi-hop-sensitive of the three tasks: each additional walk step lets a token inherit influence from a neighbor one further hop away from the current query, and 2WikiMQA is precisely the dataset where the answer-supporting token is reached through such an intermediate.
$H \in \{2, 3\}$ captures the bulk of the gain and $H = 5$ is essentially indistinguishable from $H = 3$, matching the rank-1 propagation analysis of App.~\ref{sec:theory}: after a few steps the walk converges toward the principal direction of the attention operator and additional iterations only resharpen what is already there.
We therefore default to $H = 3$, which keeps the bridge-token benefit while leaving the per-step cost negligible relative to dense attention.

\noindent\textbf{Mixing weight $\lambda$.}
$\lambda = 0$ closes most of the gap to Nexus but stays consistently below $\lambda > 0$ on 2WikiMQA, confirming that the bridge term contributes signal beyond what direct attention already captures.
$\lambda = 1.0$ overweights the bridge term and slightly hurts HotpotQA, where the answer is often a single hop away and the direct score is the cleaner signal; mixing the two terms is what makes the score robust across multi-hop structures of different depth.
The safe operating range is $\lambda \in [0.25, 0.5]$ across all three datasets, and within that range the surface is flat, so the choice does not require per-task tuning; we default to $\lambda = 0.5$ throughout.

\noindent\textbf{Observation window $W$.}
Small windows ($W \in \{8, 16, 32\}$) track full attention closely.
Larger windows degrade sharply on tasks that need short-range responsiveness: MultiFieldQA-en and 2WikiMQA both drop $\sim$5--10 points at $W = 256$, because the per-query distribution gets diluted across queries that are no longer relevant to the current decoding step.
This is a familiar effect from streaming attention-scoring methods: averaging the score over too long a window turns the per-step decision into a stale running mean and erases the locality the model relies on.
$W = 16$ is a sweet spot, small enough to remain responsive yet large enough to absorb single-step noise, and matches the window sizes reported by recent eviction work~\citep{li2024snapkv,ghadia2025morphkv}.

\noindent\textbf{Block size $b$.}
$b \in \{16, 32, 64\}$ are within noise of one another, while $b \geq 128$ degrades sharply.
The mechanism is straightforward: too-coarse blocks aggregate over too many tokens at once, hiding the single-token signal that the per-block attention distribution needs to resolve.
The flat region at small $b$ is what makes block-wise eviction practical at all: it lets implementations align the eviction granularity with the hardware-friendly block sizes already used by FlashAttention kernels without paying an accuracy tax.
We default to $b = 32$, which is the smallest block size that lets the kernel achieve full memory-bandwidth utilization on the H200.

\noindent\textbf{Density.}
Accuracy is essentially flat from $20\%$ to $80\%$ density on all three tasks, confirming that the aggressive $20\%$ setting we report in \S\ref{sec:exp-longbench}--\ref{sec:exp-ruler} does not give up accuracy headroom to less aggressive configurations.
This is the empirical signature of the marginal-mass argument of \S\ref{sec:why}: under reservoir sampling, even tight budgets retain the subtly important tokens that deterministic top-$K$ would erode, and the long-run survival probability decays gracefully with the budget rather than collapsing on the first below-cutoff step.
Practically, this means the relevant question at deployment time is not ``how much can we shrink the cache before accuracy breaks?'' but ``how much memory headroom does the workload need for other purposes?'', which is the regime modern agentic deployments actually operate in.

\section{Comparison with Head-Adaptive Budget Allocation (Ada-KV)}
\label{app:adakv}

Ada-KV~\citep{feng2025adakv} is \emph{orthogonal} to Nexus Sampling: rather than proposing a new token-importance score, it redistributes a fixed layer-wide budget \emph{across attention heads} so that heads with more concentrated attention keep more tokens.
It can therefore be layered on top of any scoring rule, including SnapKV (yielding Ada-SnapKV) or, in principle, Nexus Sampling itself.
Because this head-budget allocation is a separate axis from the scoring contribution we study in the main paper, we report Ada-SnapKV separately here rather than mixing it into Table~\ref{tab:longbench}.

Table~\ref{tab:adakv} compares Ada-SnapKV against SnapKV and Nexus Sampling in the prefill-only regime, with Full Attention shown for reference.
Ada-SnapKV improves over plain SnapKV on average, as expected from its better budget allocation, but Nexus Sampling remains competitive on the average across all three models using a uniform per-head budget, indicating that the scoring contribution we introduce is complementary to the head-adaptive allocation of Ada-KV.

\begin{table*}[ht]
    \centering
    \small
    \caption{Prefill-only LongBench accuracy at $20\%$ density, isolating the orthogonal head-adaptive budget allocation of Ada-KV~\citep{feng2025adakv}.
        Ada-SnapKV applies the Ada-KV allocation on top of the SnapKV score; Nexus Sampling uses a uniform per-head budget.
        \textbf{Bold} marks the leading method per row across the eviction methods (Full Attention shown for reference).}
    \label{tab:adakv}
    \setlength{\tabcolsep}{2.3pt}
    \resizebox{\linewidth}{!}{%
        \begin{tabular}{ll c | cccccccc cccccccc | c}
            \toprule
            \textbf{Model}                & \textbf{Method}         & \textbf{Density} & \textbf{WIKI}  & \textbf{GOV}   & \textbf{HPQA}  & \textbf{LCC}   & \textbf{MNews} & \textbf{MFQA}  & \textbf{MUS}   & \textbf{NQA}   & \textbf{COUNT} & \textbf{Retr.} & \textbf{QAS}   & \textbf{QMS}   & \textbf{REPO}  & \textbf{SamS}  & \textbf{TREC}  & \textbf{TRIV}  & \textbf{Avg.}  \\
            \midrule
            \multirow{4}{*}{Llama-3.1-8B} & Full Attention          & 100\%             & 47.75          & 34.70          & 56.96          & 55.20          & 26.76          & 56.16          & 32.77          & 29.91          & 9.66           & 99.50          & 45.06          & 25.39          & 47.80          & 43.27          & 73.00          & 92.14          & 48.50          \\
                                          & SnapKV                  & 20\%              & 47.17          & 28.69          & \textbf{58.61} & \textbf{55.42} & \textbf{23.19} & 55.44          & 31.53          & 30.28          & \textbf{10.25} & \textbf{99.50} & 40.03          & 24.60          & \textbf{48.75} & 42.17          & 68.50          & 91.43          & 47.22          \\
                                          & Ada-SnapKV              & 20\%              & 48.42          & 29.34          & 57.90          & 55.05          & 22.82          & 56.32          & \textbf{32.93} & \textbf{30.42} & 10.08          & \textbf{99.50} & \textbf{43.73} & 25.15 & 47.89          & 42.20          & 72.00          & 92.36 & 47.88          \\
                                          & \textbf{Nexus Sampling} & 20\%              & \textbf{48.82} & \textbf{30.07} & 58.03          & 54.62          & 22.65          & \textbf{57.83} & 32.68          & 29.24          & 10.00          & 99.00          & 43.47          & \textbf{25.24} & 47.87          & \textbf{42.93} & \textbf{73.00} & \textbf{92.66} & \textbf{48.01} \\
            \midrule
            \multirow{4}{*}{Llama-3.2-1B} & Full Attention          & 100\%             & 31.23          & 29.65          & 35.47          & 29.83          & 25.88          & 43.35          & 18.45          & 20.40          & 3.67           & 4.50           & 16.35          & 21.84          & 36.04          & 39.77          & 61.50          & 78.54          & 31.03          \\
                                          & SnapKV                  & 20\%              & 30.45          & 22.92          & 35.48          & 30.75          & 19.61          & 39.04          & 17.15          & \textbf{21.58} & 3.67           & \textbf{5.00}  & 15.06          & 21.18          & 35.86          & 37.40          & 58.00          & 79.31          & 29.53          \\
                                          & Ada-SnapKV              & 20\%              & 28.58          & 24.00          & \textbf{35.78} & 29.75          & 20.59          & 40.82          & \textbf{18.24} & 20.87          & 2.64           & 4.50  & \textbf{15.88} & 21.52          & \textbf{36.25} & 37.47          & \textbf{61.00} & 79.15          & 29.82          \\
                                          & \textbf{Nexus Sampling} & 20\%              & \textbf{31.86} & \textbf{24.58} & 34.52          & \textbf{30.79} & \textbf{20.91} & \textbf{41.87} & 17.53          & 20.75          & \textbf{5.67}  & \textbf{5.00}  & 15.30          & \textbf{21.74} & 36.23          & \textbf{37.71} & 58.50          & \textbf{79.95} & \textbf{30.18} \\
            \midrule
            \multirow{4}{*}{Qwen3-8B}     & Full Attention          & 100\%             & 42.62          & 33.58          & 57.94          & 57.39          & 24.82          & 52.94          & 34.33          & 27.64          & 4.50           & 100.00         & 47.91          & 23.85          & 56.67          & 44.17          & 71.50          & 90.71          & 48.16          \\
                                          & SnapKV                  & 20\%              & 42.38          & 28.76          & 56.94          & \textbf{58.07} & 20.29          & \textbf{52.76} & \textbf{33.92} & \textbf{27.96} & 6.00           & \textbf{100.00}& 43.55          & 23.17          & 56.60          & \textbf{43.90} & 66.50          & 90.71          & 46.97          \\
                                          & Ada-SnapKV              & 20\%              & 42.25          & 30.66          & \textbf{57.73} & 57.36          & 20.57          & 52.12          & 34.28          & 27.83          & 5.50           & \textbf{100.00}& \textbf{45.27} & \textbf{23.90} & \textbf{56.95} & 43.21          & 70.50          & 90.71          & 47.43          \\
                                          & \textbf{Nexus Sampling} & 20\%              & \textbf{42.49} & \textbf{31.73} & 57.68          & 56.93          & \textbf{22.11} & 52.52          & 33.73          & 27.44          & 5.00           & 99.50          & 43.72          & \textbf{23.89} & 56.23          & 43.85          & \textbf{71.00} & \textbf{91.36} & \textbf{47.45} \\
            \bottomrule
        \end{tabular}%
    }
\end{table*}

When the same head-adaptive allocation is applied on RULER, the picture is consistent: it lifts every scoring rule, and Nexus Sampling under the Ada-KV allocation (Ada-Nexus Sampling) remains the leading method on the average across all three models, confirming that our scoring contribution composes with the orthogonal budget-allocation axis.
Table~\ref{tab:adakv-ruler} reports the prefill-only RULER results with the Ada-KV allocation; for Qwen3-8B we additionally report results without the chat template, since RULER's synthetic prompts are sensitive to the template wrapping.

\begin{table}[ht]
    \centering
    \small
    \caption{Prefill-only RULER accuracy (\%) across 4K--64K context lengths at $20\%$ density with the orthogonal Ada-KV~\citep{feng2025adakv} head-adaptive budget allocation applied to each scoring rule.
        \textbf{Bold} marks the leading method per row across the eviction methods (Full Attention shown for reference).}
    \label{tab:adakv-ruler}
    \setlength{\tabcolsep}{3pt}
    \resizebox{0.7\linewidth}{!}{%
        \begin{tabular}{ll c | ccccc | c}
            \toprule
            \textbf{Model}                & \textbf{Method}              & \textbf{Density} & \textbf{4K}    & \textbf{8K}    & \textbf{16K}   & \textbf{32K}   & \textbf{64K}   & \textbf{Avg.}  \\
            \midrule
            \multirow{4}{*}{Llama-3.1-8B} & Full Attention               & 100\%            & 96.15          & 95.91          & 95.43          & 91.35          & 86.30          & 93.03          \\
                                          & Ada-SnapKV                   & 20\%             & 82.21          & 89.18          & \textbf{94.23} & \textbf{92.79} & 83.89          & 88.46          \\
                                          & Ada-PyramidKV                & 20\%             & \textbf{84.86} & 89.18          & 91.83          & 87.26          & 81.97          & 87.02          \\
                                          & \textbf{Ada-Nexus Sampling}  & 20\%             & 85.80          & \textbf{91.30} & 94.00          & 88.70          & \textbf{84.60} & \textbf{88.88} \\
            \midrule
            \multirow{4}{*}{Llama-3.2-1B} & Full Attention               & 100\%            & 76.68          & 70.67          & 65.38          & 64.66          & 60.34          & 67.55          \\
                                          & Ada-SnapKV                   & 20\%             & \textbf{57.69} & \textbf{54.33} & 50.00          & 51.68          & 44.71          & 51.68          \\
                                          & Ada-PyramidKV                & 20\%             & 56.97          & 52.64          & 43.99          & 49.52          & 38.46          & 48.32          \\
                                          & \textbf{Ada-Nexus Sampling}  & 20\%             & 57.45          & 52.88          & \textbf{52.64} & \textbf{56.49} & \textbf{50.72} & \textbf{54.04} \\
            \midrule
            \multirow{4}{*}{Qwen3-8B}     & Full Attention               & 100\%            & 97.60          & 95.43          & 88.46          & 92.79          & --             & 93.57          \\
                                          & Ada-SnapKV                   & 20\%             & 73.56          & 73.80          & 79.57          & 77.41          & --             & 76.09          \\
                                          & Ada-PyramidKV                & 20\%             & 65.38          & 66.11          & 69.95          & 74.28          & --             & 68.93          \\
                                          & \textbf{Ada-Nexus Sampling}  & 20\%             & \textbf{75.48} & \textbf{81.49} & \textbf{83.41} & \textbf{84.38} & --             & \textbf{81.19} \\
            \midrule
            \multicolumn{9}{l}{\emph{Qwen3-8B without chat template.}}                                                                                                                       \\
            \midrule
            \multirow{3}{*}{Qwen3-8B}     & Full Attention               & 100\%            & 97.60          & 95.43          & 88.46          & 92.79          & --             & 93.57          \\
                                          & Ada-SnapKV                   & 20\%             & \textbf{81.73} & \textbf{89.66} & \textbf{88.70} & \textbf{90.38} & --             & \textbf{87.62} \\
                                          & Ada-PyramidKV                & 20\%             & \textbf{81.73} & 86.78          & 85.10          & 88.94          & --             & 85.64          \\
            \bottomrule
        \end{tabular}%
    }
\end{table}

\section{Extended Related Work}
\label{sec:app_related_work}

\noindent\textbf{KV cache eviction vs.\ KV cache selection.}
Long-context LLM inference is bottlenecked by both the compute and the memory footprint of the KV cache, and two complementary lines of work try to address these axes.
\emph{KV cache eviction} permanently discards tokens under a fixed memory budget, addressing memory and compute simultaneously; an evicted token is irrecoverable.
\emph{KV cache selection}, often called \emph{sparse attention} or query-aware sparsity, keeps the full cache and only chooses which tokens to attend to per query, targeting compute alone; missed tokens at step $t$ can be re-selected at step $t+1$.
Nexus Sampling and all baselines we compare against in \S\ref{sec:experiments} are eviction methods; while representative sparse attention methods include \textbf{Quest}~\citep{tang2024quest}, \textbf{BLASST}~\citep{yuan2025blasst}, \textbf{Sketch-and-Walk}~\citep{le2026sketchwalk}, and \textbf{SOCKET}~\citep{joshi2026socket}.
A third axis is \emph{lossless} compression, e.g., \textbf{FAFO}~\citep{le2026fafo}, which drafts off a lossy compressed cache and verifies in parallel against the full cache to recover bitwise-identical generation.
Nexus Sampling is orthogonal to both: any selection method can be applied on top of the surviving cache after eviction has run, and lossless schemes like FAFO still need an underlying retention policy when the full cache no longer fits.

\noindent\textbf{KV cache eviction.}
Existing eviction methods are best understood as different choices of the \emph{score} fed to a shared deterministic top-$K$ selection rule, and have been comprehensively benchmarked across long-context tasks alongside other compression families by~\citet{yuan2024longctx}.
\textbf{StreamingLLM}~\citep{xiao2024streamingllm} hard-codes a sink-plus-recency retention rule, motivated by the attention-sink phenomenon they identify at the first few absolute positions.
\textbf{H2O}~\citep{zhang2023h2o} retains \emph{heavy hitter} tokens by ranking on cumulative attention over the generation so far, formulating eviction as a dynamic submodular problem.
\textbf{SnapKV}~\citep{li2024snapkv} computes the score against an observation window of the most recent prompt tokens, which it shows are sufficient to identify the important positions for upcoming decode steps.
\textbf{PyramidKV}~\citep{cai2024pyramidkv} introduces a layer-wise budget allocation that retains more tokens in lower layers and fewer in higher ones, motivated by the pyramidal information-funneling pattern of attention.
\textbf{Ada-KV}~\citep{feng2025adakv} adapts the budget across heads instead of layers, exploiting the empirical observation that different heads concentrate their attention on different scales of the cache.
\textbf{MorphKV}~\citep{ghadia2025morphkv} maintains a constant-size cache by correlation-aware ranking against recent tokens, iteratively refining the retained set with lightweight updates.
Despite the diversity of scoring strategies, every method above commits to the surviving tokens by deterministic top-$K$ at each step, and is therefore subject to the irreversible-verdict failure mode of \S\ref{sec:why} regardless of how good its score is.
Nexus Sampling is the first eviction method, to our knowledge, to replace the selection primitive itself rather than refine the score only.

\noindent\textbf{Reservoir sampling.}
Classical reservoir sampling for streaming uniform sampling traces to \citet{chao1982sampling} and \citet{vitter1985random}; \citet{efraimidis2006weighted} extended the scheme to weighted-without-replacement sampling via the priority $u_j^{1/w_j}$, which is exactly the per-step primitive Nexus uses.
We instantiate this rule as an eviction primitive and add the $n$-fold averaged variant of \S\ref{sec:method-ares}, which interpolates between unbiased ($n=1$) and effectively deterministic ($n \to \infty$) retention with a single knob, recovering the existing deterministic-top-$K$ template as a limit case.

\noindent\textbf{Connection to Sketch-and-Walk.}
\citet{le2026sketchwalk} introduce a multi-hop walk for sparse \emph{attention} over a full cache, propagating block importance \emph{across layers} by composing per-layer block-attention matrices through an iterative recurrence.
The Nexus walk of \S\ref{sec:method-walk} is structurally analogous but operates on an orthogonal dimension: it propagates importance \emph{across tokens}, using the per-window per-query distribution to expose bridge tokens that no single recent query attends to strongly.
The two are therefore complementary rather than competing, addressing two stages of the same pipeline: Sketch-and-Walk reduces the cost of reading from a still-full cache, while Nexus Sampling decides which tokens survive into that cache.

\end{document}